\documentclass[12pt]{article} %
\usepackage{times}
\usepackage{url}            %
\usepackage{booktabs}       %
\usepackage{amsfonts}       %
\usepackage{nicefrac}       %
\usepackage{microtype}      %
\usepackage{mkolar_definitions}
\usepackage{xspace}
\usepackage{multirow}
\usepackage{amsmath}
\usepackage{algorithm}
\usepackage{algorithmic}
\usepackage{color}
\usepackage{color, colortbl}
\usepackage{enumitem}
\usepackage{comment}
\usepackage{bm}
 \usepackage[margin=1.25in]{geometry}

\usepackage[scaled=.9]{helvet}

\usepackage{url}
\usepackage{authblk}
\usepackage{csquotes}

\usepackage[dvipsnames]{xcolor}
\usepackage[colorlinks=true, linkcolor=blue, citecolor=blue]{hyperref}

\usepackage{natbib}
\bibpunct{(}{)}{;}{a}{,}{,}

\newcount\Comments  %
\definecolor{darkgreen}{rgb}{0,0.5,0}
\definecolor{darkred}{rgb}{0.7,0,0}
\definecolor{teal}{rgb}{0.3,0.8,0.8}
\definecolor{orange}{rgb}{1.0,0.5,0.0}
\definecolor{purple}{rgb}{0.8,0.0,0.8}
\newcommand{\kibitz}[2]{\ifnum\Comments=1{\textcolor{#1}{\textsf{\footnotesize #2}}}\fi}

\definecolor{Gray}{gray}{0.9}

\newcommand{\newedit}{\color{black}}

\newcommand{\finaledit}{\color{black}}

\usepackage{amsmath,amsfonts,bm}

\def\eqref#1{equation~\ref{#1}}

\def\1{\bm{1}}

\DeclareMathAlphabet{\mathsfit}{\encodingdefault}{\sfdefault}{m}{sl}
\SetMathAlphabet{\mathsfit}{bold}{\encodingdefault}{\sfdefault}{bx}{n}

\DeclareMathOperator*{\argmax}{arg\,max}
\DeclareMathOperator*{\argmin}{arg\,min}

\usepackage{mkolar_definitions}
\ifdefined\usebigfont

\usepackage{times}
\usepackage[fontsize=13pt]{scrextend}
\AtBeginDocument{\newgeometry{letterpaper,left=1.56in,right=1.56in,top=1.71in,bottom=1.77in}}
\else
\fi

\newcommand{\ie}{\emph{i.e.}}
\newcommand{\eg}{\emph{e.g.}}

\begin{document}
\title{Offline Minimax Soft-Q-learning Under Realizability and Partial Coverage }
\author[1]{Masatoshi Uehara\thanks{uehara.masatoshi@gene.com} \footnote{This work is done when the author was at Cornell University.}} 
\author[1]{Nathan Kallus \thanks{kallus@cornell.edu}}
\author[2]{Jason D. Lee\thanks{jasonlee@princeton.edu  }}
\author[1]{Wen Sun \thanks{ws455@cornell.edu}}

\affil[1]{Genentech}
\affil[2]{Princeton University}
\date{}

\maketitle

\begin{abstract}
In offline RL, we have no opportunity to explore so we must make assumptions that the data is sufficient to guide picking a good policy, and we want to make these assumptions as harmless as possible.
In this work, we propose value-based algorithms for offline RL with PAC guarantees under just partial coverage, specifically, coverage of just a single comparator policy, and realizability of the soft (entropy-regularized) Q-function of the single policy and a related function defined as a saddle point of certain minimax optimization problem. This offers refined and generally more lax conditions for offline RL. We further show an analogous result for vanilla Q-functions under a soft margin condition. To attain these guarantees, we leverage novel minimax learning algorithms and analyses to accurately estimate either soft or vanilla Q-functions with strong $L^2$-convergence guarantees. Our algorithms' loss functions arise from casting the estimation problems as nonlinear convex optimization problems and Lagrangifying. Surprisingly we handle partial coverage even without explicitly enforcing pessimism. 
\end{abstract}

\vspace{-0.8cm}

\vspace{-0.15cm}

\section{Introduction}\label{sec:intro}

In offline Reinforcement Learning (RL), we must learn exclusively from offline data and are unable to actively interact with the environment \citep{levine2020offline}. Offline RL has garnered considerable interest in a range of applications %
where experimentation may be prohibitively costly or risky.

Offline RL is generally based on two types of assumptions:  sufficient coverage in the offline data and sufficient function approximation. For instance, classical Fitted-Q-iteration \citep{antos2008learning,chen2019information} requires (a) full coverage in the offline data, $\max_{(s,a)} d_{\pi,\mu_0}(s,a)/P_{\pi_b}(s,a)<\infty$ for any policy $\pi$ where $P_{\pi_b}(s,a)$ is the offline data's distribution on the states and actions and $d_{\pi,\mu_0}(s,a)$ is the state-action occupancy distribution under a policy $\pi$ and initial-state distribution $\mu_0(s)$; (b) realizability of the $Q^*$-function in a hypothesis class; and (c) Bellman completeness, \ie, the Bellman operator applied to any function in the hypothesis class remains in the class. Full coverage (a) and Bellman completeness (c) can be particularly stringent because offline data is often insufficiently exploratory  and Bellman completeness significantly restricts transition dynamics.

\begin{table}[!t]\label{tab:summary}
    \centering
      \caption{Summary of partial-coverage-type guarantees with model-free general function approximation. Here, $w^{\star}\coloneqq d_{\pi^{\star},\mu_0}/P_b$ where $d_{\pi^{\star},\mu_0}$ is the occupancy distribution under the optimal policy $\pi^{\star}$ starting from $\mu_0$ and $P_b$ is the distribution over the offline data. A function $\tilde w^{\star}_{\alpha}$ is a regularized marginal density ratio that satisfies $\tilde w^{\star}_0=w^{\star}$. Functions $q^{\star},q^{\star}_{\alpha},q_{\pi}$ are the optimal $Q^{\star}$-function, the soft Q-function, and the Q-function under a policy $\pi$, respectively. Functions $v^{\star}_{\alpha},l^{\star}_{\alpha}$ are Lagrange multipliers of specific minimax optimization problems. 
      The operator $\Bcal^{\pi}$ is a Bellman operator under a policy $\pi$. Function classes $\Wcal,\Qcal,\Lcal,\Vcal$ consist of functions that map states (and actions) to real numbers. Note the guarantees provided by \citet{jiang2020minimax,xie2021bellman} are more general than the below in that the output policy can compete with any policy in the policy class $\Pi$. For simplicity, we set the comparator policy to be the optimal policy $\pi^{\star}$ in this table. {Note that other studies \citep{ozdaglar2023revisiting,rashidinejad2022optimal,zhu2023importance} proposing model-free general function approximation under partial coverage rely on the completeness-type assumption as in  \citep{xie2021bellman} or realizability for any $\pi$ as in \citet{jiang2020minimax}. }
      }
      
              \scalebox{1.00}{
    \begin{tabular}{ccccc}
    \toprule
       & Primary Assumptions   \\ \midrule
    \citet{jiang2020minimax}   &  $w^{\star}\in \Wcal,\,q_{\pi}\in \Qcal\;\forall\pi\in \Pi$     \\ 
     \citet{xie2021bellman}      & $q_{\pi}\in \Qcal,\, \Bcal^{\pi} \Qcal \subset \Qcal\;\forall\pi \in \Pi$\\ 
   \citet{zhan2022offline}        & $\tilde w^{\star}_{\alpha}\in \Wcal,\,v^{\star}_{\alpha}\in \Vcal  $  &  \\
    \rowcolor{Gray}  MSQP   & 
 $q^{\star}_{\alpha}\in \Qcal,\,l^{\star}_{\alpha}\in \Lcal$  \\
  \citet{chen2022offline}      & Hard margin$,\,w^{\star}\in \Wcal,\,q^{\star} \in \Qcal $  \\ 
     \rowcolor{Gray} MQP & 
 Soft margin$,\,q^{\star}\in \Qcal,\,l^{\star}\in \Lcal$  \\ 
         \bottomrule
    \end{tabular}
    }

\end{table}

To overcome these challenges, we here propose algorithms with guarantees under realizability of single functions and refined partial coverage of single policies, and without Bellman completeness. We tackle this by introducing two novel value-based algorithms. The first algorithm, MSQP (mimimax soft-Q-learning with penalization), comprises of two steps: learning soft Q-functions {\newedit (a.k.a., entropy-regularized Q-functions, as defined in \citealp{fox2015taming,schulman2017equivalence})} from offline data, and using the softmax policies of the learned soft Q-functions. The second algorithm, MQP (mimimax Q-learning with penalization), consists of two steps: learning standard Q-functions from offline data and employing the greedy policy of  the learned Q function on the offline data. 

Using the above-mentioned two algorithms, we attain PAC guarantees under partial coverage and realizability, yet without Bellman completeness. In particular, in MSQP using soft Q-functions, we ensure strong performance under the realizability of $q^{\star}_{\alpha}$, $l^{\star}_{\alpha}$ and the (density-ratio-based) partial coverage $\max_{(s,a)}d_{\pi^\star_{\alpha},\mu_0}(s,a)/P_b(s,a)<\infty$. Here $q^{\star}_{\alpha}$ is a soft Q-function, $l^{\star}_{\alpha}$ is a function that possesses a certain dual relation to $q^{\star}_{\alpha}$, $\pi^{\star}_{\alpha}$ is the soft-max optimal policy, and {\newedit $\alpha$ is the temperature parameter for the entropy-regularization.} Notably, $\max_{(s,a)}d_{\pi^{\star}_{\alpha},\mu_0}(s,a)/P_b(s,a)<\infty$ is significantly less stringent than the uniform coverage in that the coverage is only imposed against a policy $\pi^{\star}_{\alpha}$. In MQP using Q-functions, we similarly ensure strong performance under a soft margin, the realizability of $q^{\star}$, $l^{\star}$, and the partial coverage $\max_{(s,a)}d_{\pi^{\star},\mu_0}(s,a)/P_b(s,a)<\infty$. Here $q^{\star}$ is the vanilla Q-function and $l^{\star}$ is a function that possesses a certain dual relation to $q^{\star}$, and $\pi^{\star}$ is the usual optimal policy. Note the soft margin is introduced to allow realizability on standard Q-functions rather than soft Q-functions. However, the conditions $\max_{(s,a)}d_{\pi^{\star}_{\alpha},\mu_0}(s,a)/P_b(s,a)<\infty$ or $\max_{(s,a)}d_{\pi^{\star},\mu_0}(s,a)/P_b(s,a)<\infty$ may still be strong as these marginal density ratios may not exist in large-scale MDPs. For example, this condition is easily violated when the initial distribution $\mu_0$ is not covered by $P_b$ (\ie, $\max_{s}\mu_0(s)/P_b(s)=\infty$ where $P_b(s)\coloneqq \sum_a P_b(s,a)$). Therefore, as an additional innovation, in our algorithms we can further relax these density-ratio-based partial coverage conditions. Specifically, we can demonstrate results under a \emph{refined partial coverage}, which is adaptive to Q-function classes, even when the initial distribution $\mu_0$ is not covered by $P_b$.\footnote{{\newedit Note $\mu_0$ and $P_b$ could be generally different even in the contextual bandit setting. This important setting is often considered in the literature on external validity/transportability in causal inference, as results of randomized clinical trials cannot be directly transported because covariate distributions in offline data and target data are different \citep{cole2010generalizing,pearl2022external,dahabreh2019generalizing}.}}

The primary challenge lies in the design of loss functions for effectively learning soft Q-functions and vanilla Q-functions from offline data \emph{without Bellman completness}. To tackle this, we devise new minimax loss functions with certain regularization terms to achieve favorable $L^2$-convergence rates on the offline data {\newedit (\ie, in terms of $\EE_{(s,a)\sim P_b}[\{\hat q_{\alpha}-q\}^2(s,a)]$ given an estimator $\hat q$).} This result serves as the key building block for obtaining refined partial coverage under realizability and is of independent interest in its own right. {\finaledit Existing results are often constrained to specific models, such as linear models \citep{shi2022statistical}, or they require Bellman completeness \citep{antos2008learning,chen2022well,chen2019information}. In contrast, our guarantee is applicable to any function approximation method, without the need for Bellman completeness. To the best of our knowledge, this is the first guarantee of its kind. }

Our work exhibits marked improvements over two closely related studies \citep{zhan2022offline,chen2022offline}. Similar to our work, they propose algorithms that operate under the realizability of specific functions and partial coverage, yet without Bellman completeness. \citet{zhan2022offline} ensures a PAC guarantee under (a') partial coverage in the offline data $\max_{(s,a)}d_{\tilde \pi^{\star}_{\alpha},\mu_0}(s,a)/P_b(s,a)<\infty$ where $\tilde \pi^{\star}_{\alpha}$ is a specific near-optimal policy under the regularization, which differs from the soft optimal policy, and (b') realizability of $d_{\tilde \pi^{\star}_{\alpha},\mu_0}/P_b$ and the regularized value function. However, unlike MSQP, it is unclear how to refine the abovementioned coverage, \ie, the guarantee could be vacuous when the initial distribution is not covered by offline data. A similar guarantee, but without regularization, is obtained under the additional hard margin (a.k.a., gap) condition in \citet{chen2022offline}. Our soft margin is a strict relaxation of the hard margin, which is important because, unlike the soft margin, the hard margin generally does not hold in continuous state spaces and involves very large constants in discrete state spaces. Lastly, although \citet{chen2022offline,zhan2022offline} use completely different algorithms and attain guarantees for regularized value-functions and non-regularized value functions, respectively, our guarantee can afford guarantees for regularized and non-regularized value-functions in a \emph{unified} manner since MQP can be seen as a limit of MSQP when $\alpha$ goes to $0$. %

Our contributions are summarized below and in Table~\ref{tab:summary}.

\begin{enumerate} 
    \item We establish that the optimal policy can be learned under partial coverage and realizability of the optimal soft Q-function and its dual. Notably, we abstain from the use of possibly stronger conditions in offline RL, such as full coverage, Bellman completeness, and uniform realizability over the policy class (such as $q_{\pi}\in \Qcal$ for any $\pi$ as in \citealp{jiang2020minimax}). In particular, while a similar guarantee is provided in \citet{zhan2022offline}, our partial coverage guarantee has an advantage in that we are able to potentially accommodate scenarios where the initial distribution is not covered by $P_b$. This is feasible because our algorithm is value-based in nature, which allows us to leverage the structure of the Q-function classes and refine the coverage condition.  

    \item  We demonstrate that the optimal policy can be learned under partial coverage, realizability of the Q-function and its dual, and a soft margin. %
    While a similar guarantee is obtained in \citet{chen2022offline}, our guarantee has the advantage that the soft margin is significantly less stringent than the hard margin required therein.  
\end{enumerate}

\subsection{Related Works} \label{subsec:related}

We summarize related works as follows. Further related works is discussed in \pref{sec:related}.

\paragraph{Offline RL under partial coverage.}

There is a growing number of results under partial coverage following the principle of pessimism in offline RL \citep{Yu2020,Kidambi2020}. In comparison to works that focus on tabular \citep{RashidinejadParia2021BORL,li2022pessimism,shi2022pessimistic,yin2021towards} or linear models \citep{JinYing2020IPPE,chang2021mitigating,zhang2022corruption,nguyen2022instance,bai2022pessimistic}, our emphasis is on general function approximation \citep{jiang2020minimax,uehara2021pessimistic,xie2021bellman,zhan2022offline,zhu2023importance,rashidinejad2022optimal,zanette2022bellman,ozdaglar2023revisiting}. Among them, we specifically focus on model-free methods. The representative work is summarized in Table~\ref{tab:summary}.

\paragraph{Soft (entropy-regularized) Q-functions.} 
Soft Q-functions are utilized in various contexts in RL \citep{geist2019theory,neu2017unified}. They have been shown to improve performance in online RL settings, as demonstrated in Soft Q-Learning \citep{fox2015taming,schulman2017equivalence} and Soft Actor Critic \citep{haarnoja2018soft}. In the field of imitation learning, they play a crucial role in Maximum Entropy IRL \citep{ziebart2008maximum,ziebart2010modeling}. Furthermore, within the realm of offline RL, these soft Q-functions are utilized to make the learned policy and behavior policy sufficiently similar \citep{wu2019behavior,fakoor2021continuous}. However, to the best of the authors' knowledge, none of these proposals in the context of offline RL have provided sample complexity results under partial coverage. %

\paragraph{Lagrangian view of offline RL.}

In the realm of offline policy evaluation (OPE), \citet{nachum2020reinforcement,yang2020off,huang2022beyond} have formulated the problem as a constrained linear optimization problem. Notably, within the context of policy optimization, \citet{zhan2022offline} have proposed estimators for regularized density ratios with $L^2$-convergence guarantees, which is a crucial step in obtaining a near-optimal policy. Our work is similarly motivated, but with a key distinction: our target functions are the soft Q-function and Q-function, rather than the regularized density ratio, which presents additional analytical challenges due to the nonlinear constraint.

\section{Preliminaries}

We consider an infinite-horizon discounted MDP $\Mcal =\langle \Scal,\Acal, P,r,\gamma,\mu_0\rangle$ where $\Scal$ is the state space, $\Acal$ is the finite action space, $\gamma \in [0,1)$ is the discount factor, reward $r$ is a random variable following $P_r(\cdot \mid s,a)$ on $[R_{\min},R_{\max}]$ ($R_{\min}\geq 0$), $\mu_0$ is the initial distribution. A policy $\pi:\Scal \to \Delta(\Acal)$ is a map from the state to the distribution over actions. We denote the discounted state-action occupancy distribution under a policy $\pi$ starting from an initial distribution $\mu_0$ by $d_{\pi,\mu_0}(s,a)$. With slight abuse of notation, we denote $d_{\pi,\mu_0}(s)=\sum_a d_{\pi,\mu_0}(s,a)$. We define the value under $\pi$ as $J(\pi)\coloneqq \EE_{\pi}[\sum_{t=0}^{\infty} \gamma^t \tilde r(s_t,a_t)]$ where the expectation is taken under $\pi$. We denote the optimal policy $\argmax_{\pi} J(\pi)$ by $\pi^{\star}$, and its Q-function $\EE_{\pi^{\star}}[\sum_t \gamma^t \tilde r(s_t,a_t)\mid s_0=s,a_0=a] $ by $q^{\star}(s,a)$. 

In offline RL, using offline data $\Dcal=\{(s_i,a_i,r_i,s'_i):i=1,\dots,n\}$, we search for the policy $\pi^{\star}$ that maximizes the policy value. We suppose each $(s_i,a_i,r_i,s'_i)$ is sampled {i.i.d.} from $s_i\sim P_b, a_i\sim \pi_b(\cdot\mid s), r_i\sim P_r(\cdot \mid s_i,a_i),s'_i\sim P(\cdot \mid s_i,a_i)$. We denote the sample average of $f$ by $\EE_n[f(s,a,r,s')]=\frac1n\sum_{i=1}^nf(s_i,a_i,r_i,s'_i)$, and the expectation of $f$ with respect to the offline data distribution by $\EE[f(s,a,r,s')]$ (without any scripts). The policy $\pi_b$ used to collect data is typically referred to as a behavior policy. 
With slight abuse of notation, we denote $P_b(s,a)=P_b(s)\pi_b(a\mid s)$. %

\textbf{Notation.} 
We denote the support of $P_b(\cdot)$ by $(\Scal \times \Acal)_b$, and the $L^{\infty}$-norm on $(\Scal \times \Acal)_b$ by $\|\cdot\|_{\infty,b}$. 
The $L^{\infty}$-norm on $(\Scal \times \Acal)$ is denoted by $\|\cdot\|_{\infty}$. We define $w_{\pi}(s,a)=d_{\pi,\mu_0}(s,a)/P_b(s,a)$ (if it exists). We define $\mathrm{sofmax}(h)=\frac{\exp(h(s,a))}{\sum_a \exp(h(s,a))}$ and $\|h\|_2=\EE_{(s,a)\sim P_b}[h^2(s,a)]^{1/2}$ for $h:\Scal \times \Acal \to \RR$. We denote universal constants by $c_1,c_2,\dots$. We use the convention $a/0=\infty$ when $a\neq 0$ and $0/0=0$.

\section{Algorithms}\label{sec:motivation}

In this section, we present two algorithms. {\newedit The first algorithm aims to estimate the soft optimal policy by first estimating a soft Q-function. The second algorithm estimates the optimal policy after estimating the Q-function.}

\subsection{Minimax Soft-Q-learning with Penalization}
Our ultimate aim is to mimic the optimal policy $\pi^{\star}$. As a first step, we begin by finding a policy that maximizes the following regularized objective: 
 $  \argmax_{\pi} J_{\alpha}(\pi)$ where for $\alpha>0$ we define
\begin{align*}
 J_{\alpha}(\pi)=(1-\gamma)^{-1}\EE_{(s,a)\sim d_{\pi,\mu_0},r\sim P_r(\cdot \mid s,a)}[r-\underbrace{\alpha \log \{\pi(a\mid s)/\pi_b(a\mid s)\}]}_{\text{KL penalty (between $\pi$ and $\pi_b$)}} 
\end{align*} %
This objective function is used in a variety of contexts in RL as mentioned in Section~\ref{subsec:related}. The optimal policy that maximizes $J_{\alpha}(\pi)$ with respect to $\pi$  is  
\begin{align}\label{eq:softmax_optimal}
    \pi^{\star}_{\alpha}=\mathrm{softmax}(q^{\star}_{\alpha}/\alpha + \log \pi_b),
\end{align}
where $q^{\star}_{\alpha}:\Scal \times \Acal \to \RR$ is the soft Q-function uniquely characterized by the soft Bellman equation:
\begin{align*}
\forall (s,a);\EE_{s'\sim P(\cdot \mid s,a)}[\gamma \Omega_{\alpha,\pi_b}(q^{\star}_{\alpha})(s')+r-q^{\star}_{\alpha}(s,a) \mid s,a ]=0,
\end{align*}
where $\Omega_{\alpha,\pi_b}:[\Scal\times \Acal \to \RR]\to[\Scal \to \RR]$ has $
    \Omega_{\alpha,\pi_b}(q)(s)= \alpha \log \sum_{a}\{\exp(q(s,a')/\alpha)\pi_b(a' \mid s) \}. $
As opposed to the standard objective function with $\alpha=0$, the KL penalty term serves as a regularization term that renders $\pi^{\star}_{\alpha}$ sufficiently proximate to $\pi_b$. As $\alpha$ approaches $\infty$, the optimal policy $\pi^{\star}_{\alpha}$ approaches $\pi_b$. %
On the other hand, when $\alpha=0$, $\pi^{\star}_{\alpha}$ is $\pi^{\star}$. Thus, in order to compete with $\pi^{\star}$, it is necessary to keep $\alpha$ sufficiently small. We elaborate on this selection procedure in \pref{sec:theory}. 

The natural method for offline RL using this formulation involves learning $q^{\star}_{\alpha}$ from the offline data and plugging it into \pref{eq:softmax_optimal}. The question that remains is how to accurately learn $q^{\star}_{\alpha}$ from the offline data. We consider the following optimization problem: 
\begin{align}\textstyle\label{eq:goal-pre}
    \argmin_{q \in \Qcal'} 0.5 \EE_{(s,a)\sim P_b}[q^2(s,a)]
\end{align} 
 where $\Qcal'$ consists of all functions $q: \Scal \times \Acal \to \RR$ satisfying 
 \begin{align}\label{eq:constraint1}
\EE_{s'\sim P(\cdot\mid s,a)}[\gamma \Omega_{\alpha,\pi_b}(q)(s')+r-q(s,a) \mid s,a ]=0~~~\forall (s,a)\in (\Scal \times \Acal)_b.
\end{align}
{ Here, because of the constraint \pref{eq:constraint1}, the solution is $q^{\star}_{\alpha}$. Furthermore, we use $q^2(s,a)$ in \pref{eq:goal-pre} because this choice relaxes the equality in \pref{eq:constraint1} to an inequality $\leq 0$ as we will demonstrate in  \pref{sec:intuition}. Consequently, the entire optimization problem outlined in \pref{eq:goal-pre} and \pref{eq:constraint1} transforms into a convex optimization problem. }

Then, using the method of Lagrange multipliers, \pref{eq:goal-pre} is transformed into 
\begin{align}\label{eq:goal}\textstyle
    \min_{q}\max_{l} L_{\alpha}(q,l),\quad L_{\alpha}(q,l):=
\EE\bracks{q^2(s,a)/2+ \{\gamma \Omega_{\alpha,\pi_b}(q)(s')+r-q(s,a)\}l(s,a)}. 
\end{align}

Being motivated by the above formulation, our MSQP algorithm, specified in \pref{alg:main}, approximates this formulation by replacing expectations with sample averages and restricting optimization to function classes with bounded complexity.

\begin{algorithm}[!t]
\caption{MSQP (Minimax Soft-Q-learning with Penalization) }\label{alg:main}
\begin{algorithmic}[1]
  \STATE {\bf Require}: Parameter $\alpha \in \RR^{+}$, Models $\Qcal,\Lcal \subset [\Scal \times \Acal \to \RR^{+}]$. 
    \STATE   Estimate $q^{\star}_{\alpha}$ as follows: 
    \begin{align}\label{eq:soft_q_objective}
        \hat q_{\alpha}\in \argmin_{q\in \Qcal}\max_{l\in \Lcal}\EE_n[q^2(s,a)/2+ %
        \{\gamma \Omega_{\alpha,\pi_b}(q)(s')+r-q(s,a)\}l(s,a)]. %
    \end{align}%
   \STATE Estimate the soft optimal policy: $
   \textstyle 
      \hat \pi_{\alpha}=\mathrm{softmax}(\hat q_{\alpha}/\alpha +\log \pi_b ). $
\end{algorithmic}
\end{algorithm}

\begin{remark}[Computation]
Although minimax optimization is generally difficult to solve, it is computationally feasible when we choose RKHS or linear function classes for $\Lcal$. In this case, we can solve the inner maximization problem analytically in closed form, as the objective function is linear in $l$. 
As a result, the minimax optimization problem reduces to  empirical risk minimization. 
\end{remark}

\subsection{Minimax $Q^{\star}$-learning with Penalization }

Next, we examine a policy learning algorithm utilizing $Q^{\star}$-functions. To learn $Q^{\star}$, our objective function is derived from the constrained optimization problem:
\begin{align}\textstyle\label{eq:goal2-pre}
    \argmin_{q\in \Qcal^{\star '}}0.5\EE_{(s,a)\sim P_b}[q^2(s,a)]
\end{align}
where $\Qcal^{\star '}$ consists of all functions $q: \Scal \times \Acal \to \RR$ satisfying
\begin{align*}\textstyle
\forall (s,a)\in (\Scal \times \Acal)_b;\mathbb{E}_{s'\sim P(\cdot \mid s,a)}[\gamma \max_{a'\in \Acal} q(s',a')+r-q(s,a)\mid s,a]=0. 
\end{align*}
Next, again using the method of Lagrange multipliers, \pref{eq:goal2-pre} is transformed into 
\begin{align} \label{eq:goal2}\textstyle
    \min_{q}\max_{l} L_{0}(q,l),\,L_{0}(q,l):=\EE [q^2(s,a)/2 + \{ \gamma \max_{a'}q(s',a')+r-q(s,a)\}l(s,a)]. 
\end{align}
Note $L_0$ is the limit of $L_\alpha$ as $\alpha\to0$.

Our MQP algorithm, specified in \pref{alg:main2}, similarly approximates this formulation by replacing expectations with sample averages and restricting optimization to function classes with bounded complexity. Our final policy is greedy with respect to the learned Q-function but restricting to the support of the offline data in order to avoid exploiting regions not covered by the offline data.

\begin{algorithm}[!t]
\caption{MQP (Minimax $Q^{\star}$-learning with Penalization)}\label{alg:main2}
\begin{algorithmic}[1]
  \STATE {\bf Require}:  Models $\Qcal,\Lcal \subset [\Scal \times \Acal \to \RR^{+}]$. 
    \STATE   Estimate $q^{\star}$ as follows:
    \begin{align}\label{eq:q_objective}
        \hat q_{0}\in \argmin_{q\in \Qcal}\max_{l\in \Lcal}\EE_n[q^2(s,a)/2+  %
        \{\gamma \max_{a'} q(s',a')+r-q(s,a)\}l(s,a)]. %
    \end{align}
   \STATE Estimate the optimal policy: $
       \hat \pi_{0}(a\mid s)=  \argmax_{a: \pi_b(a\mid s)>0}\hat q_{0}(s,a).$
\end{algorithmic}
\end{algorithm}

\begin{remark}[Prominent differences]\label{rem:difference_q_star}
There exist several other minimax estimators for $Q^{\star}$ including BRM \citep{antos2008learning} and MABO \citep{xie2020q}. {\newedit Although these ensure convergence guarantees in terms of Bellman residual errors, they do not ensure the guarantee in terms of $L^2$-errors, which is our focus.}

Our minimax objective function differs significantly from that of the aforementioned approaches, and its unique design plays a pivotal role in enabling $L^2$-rates.
\end{remark}

\section{$L^2$-convergence Rates for Soft $Q$-functions and $Q^{\star}$-functions}

To analyze our Q-estimators we first establish conditions that ensure $
 q^{\star}_{\alpha}=\argmin_{q \in \Qcal }\max_{l \in \Lcal} L_{\alpha}(q,l)$ on the support $(\Scal\times \Acal)_b$. Building on this, we prove $L^2$-convergence rates for $\hat q_{\alpha}$ and $\hat q_0$. {\newedit These $L^2$-convergence guarantees are subsequently translated into performance guarantees of the policies we output in \pref{sec:theory}.}

\subsection{Identification of Soft Q-functions}\label{sec:identification}

Consider an $L^2$-space $\Hcal$ where the inner product is define as $\langle h_1,h_2\rangle = \EE_{(s,a)\sim P_b}[h_1(s,a)h_2(s,a)]$.
Then we define two operators and a key function:\footnote{{\finaledit We use the notation $\cdot^{\top}$ because ${P^{\star}_{\alpha}}^{\top}$ is interpreted as the adjoint operator in the non-weighted $L^2$-space.}}
\begin{align*}
  & P^{\star}_{\alpha}   \textstyle: \Hcal \ni f \mapsto \EE_{s'\sim P(s,a),a'\sim \pi^{\star}_{\alpha}}[f(s',a')\mid (s,a)=\cdot] \in \Hcal, \\ 
   \textstyle   &  \{P^{\star}_{\alpha}\}^{\top}  \textstyle: \Hcal \ni f \mapsto \int P(\cdot \mid s,a)\pi^{\star}_{\alpha}(\cdot \mid \cdot)f(s,a)\mathrm{d}(s,a) \in \Hcal,\\
    &l^{\star}_{\alpha}(s,a) \coloneqq  \begin{cases}
         \frac{(I-\gamma \{P^{\star}_{\alpha}\}^{\top})^{-1 }(P_b(s,a)q^{\star}_{\alpha}(s,a))}{P_b(s,a)} \quad&(s,a)\in (\Scal \times \Acal)_b,  \\
         0 &(s,a)\neq (\Scal \times \Acal)_b.  
    \end{cases}
\end{align*}
These satisfy a key adjoint property, which we leverage to show $(q^{\star}_{\alpha},l^{\star}_{\alpha})$ is a saddle point of $L_{\alpha}(q,l)$.
\begin{lemma}\label{lem:adjoint}
 $\forall q\in \Hcal$, we have 
$\langle  l^{\star}_{\alpha} , (I-\gamma P^{\star}_{\alpha}) q \rangle_{\Hcal} = \langle  q^{\star}_{\alpha}, q \rangle _{\Hcal}. $
\end{lemma}

Our first assumption ensures that $l_\alpha^\star$ exists. 
\begin{assum}\label{assum:scale2}
Suppose $\|d_{\pi^{\star}_{\alpha},P_b}/P_b\|_{\infty}<\infty$. {Note the infinity norm $\|\cdot\|_{\infty}$ is over $\Scal \times \Acal$.}  
\end{assum}
\begin{proposition}\label{prop:scale2}
Under Assumption~\ref{assum:scale2}, we have
$\|l^{\star}_{\alpha}\|_{\infty} < \infty$.

\end{proposition}
\pref{prop:scale2} is immediate noting that $(I-\gamma \{P^{\star}_{\alpha}\}^{\top})^{-1}(P_b(\cdot)q^{\star}_{\alpha}(\cdot))=\sum_{t=0}^{\infty}\gamma^t (\{P^{\star}_{\alpha}\}^{\top})^t  (P_bq^{\star}_{\alpha})$ and
recalling the discounted occupancy measure under $\pi^{\star}_{\alpha}$ with initial distribution $\mu_0$ is written as $d_{\pi^{\star}_{\alpha},\mu_0} = (1-\gamma)(I-\gamma \{P^{\star}_{\alpha}\}^{\top})^{-1}(\mu_0)$. Hence, $\|l^{\star}_{\alpha}\|_{\infty}\leq (1-\gamma)^{-1}R_{\max}\|d_{\pi^{\star}_{\alpha,P_b}}/P_b\|_{\infty}$.

{\finaledit Note that $\|d_{\pi^{\star}_{\alpha,P_b}}/P_b\|_{\infty}$ crucially differs with the standard density-ratio-based concentrability coefficient  $\|d_{\pi^{*}_{\alpha},\mu_0}/P_b\|_{\infty}$ in offline RL. 
{\newedit Unlike $\|d_{\pi^{*}_{\alpha},P_b}/P_b\|_{\infty}$ , the value of $\|d_{\pi^{*}_{\alpha},\mu_0}/P_b\|_{\infty}$ can be infinite when the initial distribution $\mu_0$ is not covered by offline data $P_b$
 as the practical motivating example is explained in the footnote in Section~\ref{sec:intro} and Example~\ref{exa:external}. 
}}

Our next assumption ensures $q^{\star}_{\alpha}\geq 0$, which also guarantees that $l^{\star}_{\alpha}\geq 0$.
\begin{assum}\label{assum:scale}
Suppose $\alpha\log\|\pi^{*}_{\alpha}/\pi_b\|_{\infty}\leq R_{\min}$. %
\end{assum}
Assumption~\ref{assum:scale} can be satisfied by rescaling reward (\ie, rescaling $R_{\min}$) as long as $\|\pi^{*}_{\alpha}/\pi_b\|_{\infty}$ is finite. Hence, it is very mild.   Putting \pref{lem:adjoint} together with our assumptions we have the following.

\begin{lemma}\label{lem:saddle}
Suppose Assumptions~\ref{assum:scale2} and  \ref{assum:scale} hold. Then, $(q^{\star}_{\alpha}, l^{\star}_{\alpha})$ is a saddle point of   $L_{\alpha}(q,l)$ over $q\in \Hcal,l\in \Hcal$, \ie, $
  L_{\alpha}(q,l^{\star}_{\alpha}) \geq L_{\alpha}(q^{\star}_{\alpha},l^{\star}_{\alpha}) \geq L_{\alpha}(q^{\star}_{\alpha},l)$ $\forall q\in \Hcal, \forall l\in \Hcal$.
\end{lemma}

Recall that a point $(\tilde q,\tilde l)$ is a saddle point if and only if the strong duality holds, and $\tilde q \in \argmin_{q \in \Hcal}\sup_{l \in \Hcal}L_{\alpha}(q,l),\tilde l \in \argmax_{l \in \Hcal}\inf_{q\in \Hcal}L_{\alpha}(q,l)$ using the general characterization \citep{bertsekas2009convex}. Hence, \pref{lem:saddle} ensures $q^{\star}_{\alpha} \in  \argmin_{q \in \Hcal}\max_{l \in \Hcal}L_{\alpha}(q,l)$. 

Next, we consider the constrained optimization problem when we use function classes $\Qcal \subset \Hcal, \Lcal \subset \Hcal$. As long as the saddle point is included in $(\Qcal,\Lcal)$, we can prove that $q^{\star}_{\alpha}$ is a unique minimaxer.  %

\begin{lemma}\label{lem:saddle2}
Suppose Assumptions~\ref{assum:scale2} and \ref{assum:scale} hold, $q^{\star}_{\alpha} \in \Qcal$, and $l^{\star}_{\alpha}\in \Lcal$. Then, we have that
$q^{\star}_{\alpha} =  \argmin_{q \in \Qcal}\sup_{l \in \Lcal}L_{\alpha}(q,l)$ on the support $(\Scal \times \Acal)_b$. 
\end{lemma}

This establishes that realizability ($q^{\star}_{\alpha} \in \Qcal,\,l^{\star}_{\alpha}\in \Lcal$)
is sufficient to identify $q^{\star}_{\alpha}$ on the offline data distribution. At a high level, $q^{\star}_{\alpha} \in  \argmin_{q \in \Qcal}\sup_{l \in \Lcal}L_{\alpha}(q,l)$ is established through the invariance of saddle points, \ie, saddle points over original sets remain saddle points over restricted sets. Its uniqueness is verified by the strong convexity in $q$ of $L_{\alpha}(q,l)$ induced by $\EE_{(s,a)\sim P_b}[q^2(s,a)]$.

\subsection{$L^2$-convergence Rate for   
Soft Q-estimators}\label{subsec:finite}

Based on the population-level results in \pref{sec:identification}, we give a finite-sample error analysis of $\hat q_\alpha$  
\begin{assum}[Realizability of soft Q-function]\label{assum:softmax}
   Suppose $q^{\star}_{\alpha} \in \Qcal$ and $\|q\|_{\infty}\leq B_{\Qcal}\,\forall q\in \Qcal$.
\end{assum}

\begin{assum}[Realizability of Lagrange multiplier]\label{assum:largrange}
Suppose $l^{\star}_{\alpha} \in \Lcal$ and $\|l\|_{\infty}\leq B_{\Lcal}\,\forall l \in \Lcal$.
\end{assum}

It is natural to set $B_{\Qcal}=(1-\gamma)^{-1}R_{\max}$ and $B_{\Lcal}=(1-\gamma)^{-1}R_{\max}\|d_{\pi^*_{\alpha},P_b}/P_b\|_{\infty}$, but letting these be arbitrary offers further flexibility to our results.

\begin{theorem}[$L^2$-convergence of soft Q-estimators]\label{thm:convergence_soft}
Suppose Assumptions~\ref{assum:scale2}, \ref{assum:scale}, \ref{assum:softmax}, and \ref{assum:largrange} hold. Then, with probability $1-\delta$, the $L^2$-error $\|\hat q_{\alpha}-q^{\star}_{\alpha}\|_2$ is upper-bounded by $$
 c \prns{  \Bcal^2_{\Qcal}+\Bcal_{\Qcal}\Bcal_{\Lcal}\{\alpha+\ln(|\Acal|)\}}^{1/2} \left(  \ln(|\Qcal||\Lcal|/\delta)/n\right)^{1/4}.  $$
\end{theorem}

Our result is significant as it relies on realizability-type conditions rather than Bellman closedness. Since the majority of existing works focus on non-regularized Q-functions, we postpone the comparison to these existing works to the next section.
 Note when $\Qcal$ and $\Lcal$ are infinite, we can easily replace $|\Qcal|,|\Lcal|$ with their $L^{\infty}$-covering numbers following \citet{uehara2021finite}. Details are given in the appendix.

\subsection{$L^2$-convergence Rate for  $Q^{\star}$-functions}

Next, we give analogous finite-sample error analysis of $\hat q_0$ leveraging the same reasoning.  

\begin{assum}[Realizability of $Q^{\star}$-functions]\label{assum:q_star}
Suppose $q^{\star}\in \Qcal$ and $\|q\|_{\infty}\leq B_{\Qcal}\,\forall q\in \Qcal$. 
\end{assum}

Next, we define the Lagrange multiplier:
\begin{align*}
   \textstyle \{P^{\star}\}^{\top}&:\Hcal \ni f \mapsto \int P(\cdot\mid s,a)\pi^{\star}(\cdot\mid \cdot)f(s,a)\mathrm{d}\mu(s,a)\in \Hcal, 
\\
l^{\star} &\coloneqq \{(I-\gamma \{P^{\star}\}^{\top})^{-1 }(q^{\star}P_{\pi_b})\}/P_{\pi_b}.
\end{align*}
{\finaledit While $l^{\star} $ involves the density ratio, this is always well-defined as long as $\|d_{\pi^{\star},{P_b}}/P_b\|_{\infty}<\infty$. %
} Then, it can be similarly established that $(q^{\star},l^{\star})$ is a saddle point of $L_0(q,l)$ over $q\in \Hcal,l\in \Hcal$ as we show in \pref{lem:saddle}. We lastly require its realizability.

\begin{assum}[Realizability of Lagrange multiplier]\label{assum:lagrange2}
Suppose  $\|d_{\pi^{\star},{P_b}}/P_b\|_{\infty}<\infty$ and $l^{\star} \in \Lcal$. Further suppose $\|l\|_{\infty}\leq B_{\Lcal}\,\forall l\in \Lcal$. 
\end{assum}

\begin{theorem}[$L^2$-convergence of Q-estimators]\label{thm:convergence_q}
Suppose Assumptions~\ref{assum:q_star} and \ref{assum:lagrange2} hold. Then, with probability $1-\delta$, the $L^2$-error $\|\hat q_{0}-q^{\star}\|_2$ is upper-bounded by $$
 c\prns{  \Bcal^2_{\Qcal}+  \Bcal_{\Lcal}\Bcal_{\Qcal}}^{1/2} (\ln(|\Qcal||\Lcal|/\delta)/n)^{1/4}.$$
\end{theorem}

{\finaledit To the best of our knowledge, this is the first guarantee on $L^2$ errors for learning $q^*$ using \emph{general function approximation without relying on Bellman completeness}. This is highly nontrivial, and we have carefully crafted our algorithm to obtain this guarantee. Existing results are often specific to particular models, such as linear models \citep{shi2022statistical}, or they require Bellman completeness \citep{chen2019information,chen2022well}, or they are limited to offline policy evaluation scenarios \citep{huang2022beyond} (\ie, cases involving linear Bellman operators, but nonlinear Bellman operators). Actually, it seems that even under the assumption of Bellman completeness, obtaining an L2 guarantee \emph{without strong coverage assumptions} remains unclear. A detailed comparison among these different approaches is presented in Section \ref{sec:contribution_over}.
}

\section{Finite Sample Guarantee of MSQP}\label{sec:theory}

In this section, we present our primary sample complexity guarantee for our  MSQP algorithm under the assumptions of realizability of  $q^{\star}_{\alpha}$ and $l^{\star}_{\alpha}$  and partial coverage. We first show the learned policy $\hat \pi_{\alpha}$ can compete with $\pi^{\star}_{\alpha}$.  Finally we show  $\hat \pi_{\alpha}$ can compete with $\pi^{\star}$ by selecting $\alpha$ properly.   

We first introduce the flattened behavior policy $\pi^{\diamond}_b$, which is uniform on the support of $\pi_b$.
We use it as a technical device to define a model-free concentrability coefficient following \citet{xie2021bellman}. 
\begin{definition}[Model-free concentrability coefficient]
Define
\begin{align*}
    C_{\Qcal,d_{\pi^{\star}_{\alpha},\mu_0} }\coloneqq \sup_{q\in \Qcal}\frac{\EE_{s \sim d_{\pi^{\star}_{\alpha},\mu_0},a\sim \pi^{\diamond}_b(a\mid s) }[\|q(s,a) -q^{\star}_{\alpha}(s,a)\|^2_2 ]}{\EE_{(s,a)\sim P_b}[\|q(s,a) -q^{\star}_{\alpha}(s,a)\|^2_2]}
\end{align*}
where $\pi^{\diamond}_b(\cdot \mid s)=\begin{cases}
     0  \quad & \pi_b(\cdot \mid s) = 0 \\
     1/|\{a \in \Acal\mid \pi_b(a \mid s)>0\}|  \; & \pi_b(\cdot \mid s)>0
    \end{cases}$ is the flattened behavior policy.
\end{definition}
Clearly, $C_{\Qcal,d_{\pi^{\star}_{\alpha},\mu_0} }$ is smaller than density-ratio-based concentrability coefficient, in other words, 
$$C_{\Qcal,d_{\pi^{\star}_{\alpha},\mu_0} }\leq \max_{(s,a)} \frac{d_{\pi^{\star}_{\alpha},\mu_0}(s)\pi^{\diamond}_b(a\mid s)}{P_b(s)\pi_b(a\mid s)}.
$$ 
Here, we always have $\|\pi^{\diamond}_b/\pi_b\|< \infty$ even if $\pi_b(a\mid s)$ is $0$ for some $(s,a)$. In the special case where $\pi_b(a\mid s)\geq 1/C'$ for any $(s,a)$, we have $C_{\Qcal,d_{\pi^{\star}_{\alpha},\mu_0} }\leq C'\|d_{\pi^{\star}_{\alpha},\mu_0}/P_b\|_{\infty} $. 
The coefficient $C_{\Qcal,d_{\pi^{\star}_{\alpha},\mu_0} }$ is is a refined concentrability coefficient, which adapts to a function class $\Qcal$. For example, in linear MDPs, it reduces to a relative condition number as follows. Similar properties are obtained in related works \citep{xie2021bellman,uehara2021pessimistic}. 

\begin{example}[Linear MDPs]
A linear MDP is one such that, for a known feature vector $\phi:\Scal \times \Acal \to \RR^d$, the true density satisfies $P(s' \mid s,a)=\langle \mu^{\star}(s'),\phi(s,a)\rangle $ for some $\mu^{\star}:\Scal \to \RR^d$ and the reward function satisfies $\EE[r\mid s,a]=\langle \theta_r, \phi(s,a) \rangle $ for some $\theta_r \in \RR^d$. 

In linear MDPs, $q^{\star}_{\alpha}$ is clearly linear in $\phi(s,a)$. Hence, the natural function class is $\Qcal=\{\langle \theta,\phi(s,a)\rangle \mid \|\theta\|\leq B\}$ for a certain $B\in \mathbb{R}^{+}$. Then, we have 
\begin{align*}
C_{\Qcal,d_{\pi^{\star}_{\alpha},\mu_0} }=\sup_{x\neq 0}\frac{x^{\top}\EE_{s \sim d_{\pi^{\star}_{\alpha},\mu_0}, a\sim \pi^{\diamond}_b(a\mid s) }[\phi(s,a)\phi(s,a)^{\top}]x  }{x^{\top}\EE_{(s,a)\sim P_b}[\phi(s,a)\phi(s,a)^{\top}]x}. 
\end{align*}
\end{example}

We are now prepared to present our main result, which states that given the realizability of the soft Q-function $q^{\star}_{\alpha}$ and Lagrange multiplier $l^{\star}_{\alpha}$, it is possible to compete with $\pi^{\star}_{\alpha}$ under the coverage condition $C_{\Qcal,d_{\pi^{\star}_{\alpha},\mu_0} }<\infty,\|d_{\pi^{\star}_{\alpha},P_b}/P_b\|_{\infty}<\infty $. 

\begin{theorem}[$\hat \pi_{\alpha}$ can compete with $\pi^{\star}_{\alpha}$]\label{thm:main}
Fix $\alpha>0$. Suppose Assumptions \pref{assum:scale2}, \pref{assum:scale}, \pref{assum:softmax}, and \pref{assum:largrange} hold. 
With probability $1-\delta$, the regret $ J(\pi^{\star}_{\alpha})-J(\hat \pi_{\alpha})$  is upper-bounded by 
\begin{align*}
n^{-1/4}\mathrm{Poly}\left(|\Acal|,\Bcal_{\Qcal},\Bcal_{\Lcal},C_{\Qcal,d_{\pi^{\star}_{\alpha},\mu_0} }, \ln\left(\frac{|\Qcal||\Lcal|}{\delta}\right),\frac{1}{1-\gamma},R_{\max}\right)
\end{align*}

\end{theorem}
The proof mainly consists of two steps: (1) obtaining $L^2$-errors of $\hat q_{\alpha}$ as previously demonstrated in \pref{thm:convergence_soft}, (2) translating this error into the error of $\hat \pi_{\alpha}$. In the second step, the Lipshitz continuity of the softmax function plays a crucial role. If there is no regularization ($\alpha=0$) and the greedy policy of $q^{\star}_{0}$ is utilized, the second step does not proceed (without any further additional assumptions). 

Our ultimate goal is to compete with $\pi^{\star}$. 
\pref{thm:main} serves as the primary foundation for this goal. The remaining task is to analyze the approximation error $J(\pi^{\star})-J(\pi^{\star}_{\alpha})$. Fortunately, this term can be controlled through  $\alpha$ and the density ratio between $\pi^{\star}$ and $\pi_b$. Then, by properly controlling $\alpha$, we can obtain the following sample complexity result.
 
\begin{theorem}[PAC guarantee of $\hat \pi_{\alpha}$]\label{cor:main}
Fix any $\epsilon>0$. Suppose Assumptions~\pref{assum:scale2}, \pref{assum:scale}, \pref{assum:softmax}, and \ref{assum:largrange} hold for $\alpha=c/n^{1/8}$ and  $\|\pi^{\star}_{0}/\pi_b\|_{\infty}\leq C_{0}$, $C_{\Qcal,d_{\pi^{\star}_{\alpha},\mu_0} }<\infty$. Then, if $n$ is at least  
{
\begin{align*}
\epsilon^{-8}\mathrm{Poly}(|\Acal|,\Bcal_{\Qcal},\Bcal_{\Lcal},C_{\Qcal,d_{\pi^{\star}_{\alpha},\mu_0} }, \ln(|\Qcal||\Lcal|/\delta),%
(1-\gamma)^{-1},\ln(C_0),R_{\max}), 
\end{align*}
}
with probability at least $1-\delta$, we can ensure $ J(\pi^{\star})-J(\hat \pi_{\alpha})\leq \epsilon$. 
\end{theorem}

{\finaledit In summary, the realiazability of $q^{\star}_{\alpha},l^{\star}_{\alpha}$, per-step coverage $\|\pi^{\star}_0/\pi_b\|_{\infty}<\infty$ and partial coverage $C_{\Qcal,d_{\pi^{\star}_{\alpha},\mu_0} }<\infty,\|d_{\pi^{\star}_{\alpha},P_b}/P_b\|_{\infty}<\infty $ are sufficient to compete with $\pi^{\star}$.} This is a novel and attractive result. Firstly, if we solely use the na\"ive FQI or Bellman residual minimization, existing PAC results require the global coverage $\|d_{\pi,\mu_0}/P_b\|_{\infty}<\infty$ for any possible policy $\pi$ \citep{munos2008finite,antos2008learning}. Our result only requires coverage under a single policy $\pi^{\star}_{\alpha}$ (near-optimal policy). Secondly, we only require the realizability of two functions, and we do not necessitate realizability-type conditions for all policies in the policy class or Bellman completeness, unlike existing works with partial coverage \citep{xie2021bellman,jiang2020minimax}.

The most similar result is \citet{zhan2022offline}. However, our guarantee possesses a certain advantage over their guarantee as follows. They demonstrate the realizability of certain functions $\tilde w^{\star}_{\alpha},v^{\star}_{\alpha}$ and partial coverage $\|d_{\tilde\pi_{\alpha},\mu_0}/P_b\|< \infty$ are sufficient conditions in offline RL, where $\tilde w^{\star}_{\alpha}=d_{\tilde\pi_{\alpha},\mu_0}/P_b$ ($\tilde\pi_{\alpha}$ is a certain regularized optimal policy, but fundamentally distinct from $\pi^{\star}_{\alpha}$) and $v^{\star}_{\alpha}$ is a near-optimal regularized value function parameterized by $\alpha$. Here, we have  $\tilde w^{\star}_{0}=w^{\star},v^{\star}_{0}=v^{\star}$. Our guarantee has a similar flavor in the sense that it roughly illustrates realizability and partial coverage are sufficient conditions. 
However, the meanings of realizability and partial coverage are significantly different. In particular, by employing our algorithm, we can ensure PAC guarantees under the boundedness of the refined concentrability coefficient $C_{\Qcal,d_{\pi^{\star}_{\alpha},\mu_0}}<\infty$ (and $\|d_{\pi^{\star}_{\alpha},P_b}/P_b\|_{\infty}$ through $\Bcal_{\Lcal}$). 
As a result, the $L^{\infty}$-norm of the density-ratio-based concentrability coefficient $\|d_{\pi^{\star}_{\alpha},\mu_0}/P_b\|_{\infty}$ can even be infinite. More specifically, we can permit situations where $\max_{s}\mu_0(s)/P_b(s)=\infty$ as we will see the practical example soon. Conversely, \citet{zhan2022offline} excludes this possibility since the algorithm explicitly estimates the density ratio $\tilde w^{\star}_{\alpha}$.

\begin{example}[Contextual bandit under external validity]
 \label{exa:external}
 We consider the contextual bandit setting where we want to optimize $J(\pi)=\EE_{s\sim \mu_0,a\sim \pi(s),r\sim P_r(s,a)}[r]$ using offline data $s \sim P_b,a\sim \pi_b(s),r\sim P(s,a)$. This is the simplest RL setting with $\gamma=0$. Here, note $\mu_0$ could be different from $P_b$. This case often happens in practice as discussed in the literature on causal inference related to external validity \citep{pearl2022external,dahabreh2019generalizing,uehara2020off}, which refers to the shift between the target population and the offline data. Here, our PAC guarantee does not require that $\mu_0(s)$ is covered by $P_b(s)$ in terms of the density ratio as long as the relative condition number is upper-bounded when we use linear models. On the other hand, \citet{zhan2022offline} excludes this possibility.  
\end{example}

{\finaledit Despite the aforementioned advantage of our approach, unfortunately, our sample complexity of $O(1/\epsilon^8)$ is slower compared to that of $O(1/\epsilon^6)$ in \citet{zhan2022offline}. In the following, we demonstrate that MQP, which is a special version of MSQP when $\alpha \to 0$, can achieve a faster rate of $O(1/\epsilon^2)$.}

\section{Finite Sample Guarantee of MQP} \label{sec:margin}

In this section, building upon the convergence result of $\hat q_0$, we demonstrate the finite sample guarantee of our MQP algorithm under partial coverage. We first introduce the soft margin.

\begin{assum}[Soft margin]\label{assum:margin}
For any $a'\in \Acal$, there exists $t_0 \in \mathbb{R}^{+},\beta \in (0,\infty]$ such that 
\begin{align*}
    \PP_{s \sim d_{{\pi^{\star},\mu_0}}}(0<|q^{\star}(s,\pi^{\star}(s))-q^{\star}(s,a')|<t)\leq (t/t_0)^{\beta} 
\end{align*}
for any $t>0$. Here, we use the convention $x^{\infty}=0$ if $0<x<1$ and $x^{\infty}=\infty$ if $x>1$. 
\end{assum}

In the extreme case, if there exists a gap in $q^{\star}$ (also known as a hard margin) so that the best action is always better than the second-best by some lower bounded amount, then the soft margin is satisfied with $\beta=\infty$. Thus, the soft margin is more general than the gap condition  used in \citet{simchowitz2019non,wu2022gap}.
Crucially, a gap generally does \textit{not} exist in continuous state spaces unless Q-functions are discontinuous or one action is always option, or a gap involves a large $t_0$ constant in discrete state spaces with bad dependence on the number of states. In contrast, a soft margin with some $\beta>0$ generally holds (see, \eg, lemma 4 in \citealp{hu2021fast}). The soft margin is widely used in the literature on classification, decision making, and RL \citep{audibert2007fast,perchet2013multi,luedtke2020performance,hu2021fast,hu2022fast}.

\begin{theorem}[PAC guarantee of $\hat \pi_0$]\label{thm:pac_q_function}
Suppose Assumptions~\ref{assum:q_star}, \ref{assum:lagrange2}, and \ref{assum:margin} hold and  $\|\pi^{\star}/\pi_b\|_{\infty}\leq C_0$. 
 Fix any $\epsilon>0$. Then, if $n$ is at least 
\begin{align*}
\textstyle
\{\frac{|\Acal|}{\epsilon}\}^{\frac{4+2\beta}{\beta}} \mathrm{Poly}\left(t^{-1}_0,|\Acal|,\Bcal_{\Qcal},\Bcal_{\Lcal},C_{\Qcal,d_{\pi^{\star},\mu_0} }, \ln\left (\frac{|\Qcal||\Lcal|}{\delta}\right) %
,(1-\gamma)^{-1},\ln(C_0),R_{\max} \right)%
\end{align*}
with probability at least $1-\delta$, we can ensure $ J(\pi^{\star})-J(\hat \pi_{0})\leq \epsilon$. 
\end{theorem}

The proof mainly consists of two steps: (1) obtaining $L^2$-errors of $\hat q_{0}$ as demonstrated in \pref{thm:convergence_q}, (2) translating this error into the error of $\hat \pi_{0}$. In the second step, the soft margin plays a crucial role.

These theorems indicate that the realizability of the $Q$-function $q^{\star}$ and Lagrange multiplier $l^{\star}$, and the soft margin are sufficient for the PAC guarantee under partial coverage $C_{\Qcal,d_{\pi^{\star},\mu_0} }<\infty$ {\finaledit , $\|d_{\pi^{\star},P_b}/P_b\|_{\infty}<\infty$ }. Our algorithm is \emph{agnostic} to $\beta$ and operates under any value of $\beta$. {\finaledit In particular, when there is a gap ($\beta=\infty$), we can achieve sample complexity of $O(1/\epsilon^2)$ \footnote{{Similar to the findings in \citet{wang2022gap}, in general offline RL, we may potentially achieve a result of $O(1/\epsilon)$. We leave room for further enhancements in future research. }
}.} In comparison to \pref{thm:main}, although we additionally use the soft margin, the realizability in \pref{thm:pac_q_function} is more appealing since it is imposed on the standard Q-function $q^{\star}$. 
The closest guarantee to our work can be found in \citet{chen2022offline}, which demonstrates that the existence of the gap in $q^{\star}$, the realizability of $q^{\star},w^{\star}(:=d_{\pi^{\star},\mu_0}/P_{\pi_b})$, and partial coverage $\|w^{\star}\|_{\infty}< \infty$ are sufficient conditions. { A similar comparison is made in \citet{ozdaglar2023revisiting}.} In comparison to their work, we use the soft margin, which is significantly less stringent.

\section{Conclusions}
 \vspace{-0.2cm}

We propose two value-based algorithms, MSQP and MQP, that operate under realizability of certain functions and partial coverage (\ie, single-policy-coverage). Notably, our guarantee does not require Bellman completeness and uniform-type realizability over the policy class. While guarantees with similar flavors are obtained in \citet{zhan2022offline,chen2022offline}, MSQP can potentially relax the density-ratio-based partial coverage regarding the initial distribution as opposed to \citet{zhan2022offline}, and MQP can operate under the soft-margin, which is less stringent than the hard margin imposed in \citet{chen2022offline}. Moreover, both algorithms work on Q-functions, which are more commonly used in practice.

\bibliography{ref}

\appendix

{

\section{More Related Works} \label{sec:related}

\label{sec:contribution_over}

We elaborate on the challenge and novelty of our guarantees in comparison to  previous related works.

\paragraph{Existing $L^2$-guarantees for offline RL \citep{shi2022statistical}.}

\citet{shi2022statistical} derived the $L^2$ convergence rate for learning the optimal $Q^{\star}$-function. However, their result is restricted to linear models. While linear models allow for converting Bellman residual errors into $L^2$-errors under mild assumptions regarding the non-singularity of the covariance matrix, this guarantee can be achieved through various methods such as modified BRM \citep{antos2008learning}. However, it remains unclear how to extend this guarantee to general function approximation settings.

\paragraph{Existing $L^2$-guarantees for off-policy evaluation \citep{huang2022beyond,zhan2022offline}.}

According to \citet{huang2022beyond}, $L^2$-rates were obtained in the context of offline policy evaluation, while \citet{zhan2022offline} obtained the $L^2$-rate for the (regularized) optimal marginal density ratio function. While our approach builds upon these works, it is still considered novel due to the nonlinear constraints we consider, which provide an additional challenge in the analysis. In contrast, the constraints considered by \citet{huang2022beyond,zhan2022offline} were linear.

In order to grasp the difficulty at a high level, suppose that the state space is tabular \footnote{Note our theory still proceeds in the non-tabular case.}. Recall our estimator is motivated by the constrained optimization problem \pref{eq:constraint1}. Although the left-hand side of \pref{eq:constraint1} is convex in $q$, the constraint induced by the equality is generally not convex. As a result, this problem is not a convex optimization problem. Therefore, it is not straightforward whether the saddle point exists when we consider the minimax form. This differs from previous works \citep{huang2022beyond,zhan2022offline}, where the constraint from the linear equality remains convex. 

The key insight for addressing this issue is the realization that the equality constraint ($=$) in \pref{eq:constraint1} can be relaxed to an inequality constraint ($\leq$) in \pref{eq:constraint1}. This relaxation is surprisingly valid in our problem. As a result, the resulting optimization problem is convex, and a saddle point exists. For further discussion, see \pref{sec:intuition}.

\vspace{-0.2cm}
\paragraph{Bellman residual minimization \citep{antos2008learning,chen2019information,chen2022well}.} \citet{chen2019information} derived convergence rates for Q-functions in terms of Bellman residual errors: $\EE_{(s,a)\sim P_b}[(\Bcal^\star q-q)^2(s,a)]$ where $\Bcal^{\star}$ is a Bellman operator. While the convergence of the $L^2$-error implies that of the Bellman residual error since we have 
\begin{align*}
&\EE_{(s,a)\sim P_b}[(\Bcal^{\star} q-q)^2(s,a)]\leq %
(1+\gamma^2 \EE_{(s,a)\sim P_b}[(\pi^{\star}/\pi_b)^2(s,a)])\|q-q^{\star}\|^2_2, 
\end{align*}
the reverse direction does \emph{not} generally hold. Although \citet{chen2022well} demonstrated the reverse direction by postulating several potentially stringent conditions \footnote{These conditions would implicitly impose restriction of the coverage of the offline data.} of offline data regarding OPE, it does not generally hold in the absence of these assumptions. Furthermore, we consider the more challenging setting, offline policy optimization where the Bellman operator is nonlinear. Thus, it is imperative to utilize our specialized objective function to ensure a stronger guarantee in terms of $L^2$-errors.

\vspace{-0.25cm}

\section{Intuition of Why Saddle Points Exist}\label{sec:intuition}

In this section, in the tabular setting, we explain how the original optimization problem is reduced to a convex optimization problem in detail. Note proofs in our theorems do not directly use the facts in this \pref{sec:intuition}. Hence, they hold even in the non-tabular setting.

\subsection{$Q^{\star}$-functions }

We consider the optimization problem: 
\begin{align*}
    \argmin_{q\in \Qcal^{\star '}}0.5\EE_{(s,a)\sim P_b}[q^2(s,a)]
\end{align*}
where $\Qcal^{\star '}$ is a space that consists of $q\in ([\Scal \times \Acal]\to \RR)$ that satisfies
\begin{align*}
    \forall (s,a)\in (\Scal \times \Acal)_b; \mathbb{E}_{s'\sim P(\cdot \mid s,a)}[\gamma \max_{a'\in \Acal} q(s',a')+r-q(s,a)\mid s,a]=0. 
\end{align*}
The solution (on the offline data) is clearly $q^{\star}_{0}$. However, this optimization problem is not a convex optimization problem. Hence, the associated minimax optimization might not have a saddle point. 

Here, we want to claim the following:
\begin{align}\label{eq:convex}
    q^{\star}_{0}(s,a) = \argmin_{q\in \Qcal^{\star '}}0.5\EE_{(s,a)\sim P_b}[q^2(s,a)] (\forall (s,a)\in (\Scal\times \Acal)_b)
\end{align}
where $\Qcal^{\star '}$ is a space that consists of $q\in ([\Scal \times \Acal]\to \RR)$ that satisfies
\begin{align*}
    \forall (s,a)\in (\Scal \times \Acal)_b; \mathbb{E}_{s'\sim P(\cdot \mid s,a)}[\gamma \max_{a'\in \Acal} q(s',a')+r-q(s,a)\mid s,a]\leq 0. 
\end{align*}
This is a convex optimization problem. Hence, we can expect the associated minimax optimization would have a saddle point. In fact, this is proved as follows. 

\begin{lemma}
The equation \pref{eq:convex} holds. 
\end{lemma}

\begin{proof}
Suppose $q(s,a;c)$ is a minimum $L^2$-norm solution that satisfies 
\begin{align*}
   \forall (s,a) \in (\Scal \times \Acal)_b; \mathbb{E}_{s'\sim P(\cdot \mid s,a)}[\gamma \max_{a'\in \Acal} q(s',a')+r-q(s,a)\mid s,a]=-c(s,a)
\end{align*}
 and $c(s,a)\geq 0$. 

 Suppose $c(s,a)\neq 0$ on some $(s,a) \in (\Scal \times \Acal)_b$. Then, we obtain 
 \begin{align*}
     \forall (s,a) \in (\Scal \times \Acal)_b;   q(s,a;c) =  \EE_{\pi^{\star}(c)}[r_t + c(s_t,a_t) \mid (s_0,a_0)=(s,a)]
 \end{align*}
 where $ \pi^{\star}(c)$ is the optimal policy when the reward is $\tilde r(s,a)+c(s,a)$. This is derived by a contraction mapping theorem on the $L^{\infty}$-space on $(\Scal \times \Acal)_b$.  Here, recall 
 \begin{align*}
    \forall (s,a) \in (\Scal \times \Acal)_b;  q^{\star}(s,a) = q(s,a;0)\geq 0 
 \end{align*}
Furthermore, $ \forall (s,a) \in (\Scal \times \Acal)_b;q(s,a;0)\leq q(s,a;c)$ clearly holds, and, especially, the strict inequality holds on some point in $(\Scal \times \Acal)_b$. This is because if $q(s,a;0)= q(s,a;c)$, we get the contradiction:
\begin{align*}
      \forall (s,a) \in (\Scal \times \Acal)_b; 0 = \mathbb{E}_{s'\sim P(\cdot \mid s,a)}[\gamma \max_{a'\in \Acal} q(s',a')+r-q(s,a)\mid s,a]=c(s,a). 
\end{align*}

However, in this situation, we have 
\begin{align*}
    \EE_{(s,a)\sim P_b}[q^2(s,a;0) ]<   \EE_{(s,a)\sim P_b}[q^2(s,a;c) ]. 
\end{align*}
Thus, it contradicts $q^2(s,a;c)$ takes a least $L^2$-norm. Therefore, it is concldued $c(s,a)=0$ for any $(s,a) \in (\Scal \times \Acal)_b$. Thus, using the contraction mapping theorem, this implies the least $L^2$-norm solution is unique and $q^{\star}(s,a)$ on $(s,a) \in (\Scal \times \Acal)_b$.

\end{proof}

\subsection{Soft $Q$-functions }

We consider the optimization problem: 
\begin{align}\label{eq:convex2}
    \argmin_{q \in \Qcal'} 0.5 \EE_{(s,a)\sim P_b}[q^2(s,a)]
\end{align} 
 where $\Qcal'$ is a space that consists of $q\in ([\Scal \times \Acal]\to \RR)$ that satisfies
 \begin{align}\label{eq:constraint12}
  \forall (s,a)\in (\Scal \times \Acal)_b;  \EE_{s'\sim P(\cdot\mid s,a)}[\gamma \alpha \log \sum_{a'}\{\exp(q(s',a')/\alpha) \pi_b(a'\mid s')\}+r-q(s,a) \mid s,a ]\leq 0. 
\end{align}

\begin{lemma}
Suppose $\alpha \log \|\pi^{\star}/\pi_b\|_{\infty} \leq R_{\min}$. Then, the solution to \pref{eq:convex2} is $q^{\star}_{\alpha}$.  
\end{lemma}

\begin{proof}
Suppose $q_{\alpha}(s,a;c)$ is a minimum $L^2$-norm solution that satisfies 
\begin{align*}
   \forall (s,a) \in (\Scal \times \Acal)_b; \mathbb{E}_{s'\sim P(\cdot \mid s,a)}[\gamma \alpha \log(\sum_{a'} \exp(q(s',a')/\alpha)\pi_b(a'\mid s')+r-q(s,a)\mid s,a]=-c(s,a)
\end{align*}
 and $c(s,a)\geq 0$. 

 Suppose $c(s,a)\neq 0$ on some $(s,a) \in (\Scal \times \Acal)_b$. Then, we obtain 
 \begin{align*}
     \forall (s,a) \in (\Scal \times \Acal)_b;   q_{\alpha}(s,a;c) =  \EE_{\pi^{\star}_{\alpha}(c)}[r_t + c(s_t,a_t) \mid (s_0,a_0)=(s,a)]
 \end{align*}
 where $ \pi^{\star}_{\alpha}(c)$ is the optimal policy when the reward is $\tilde r(s,a)+c(s,a)$. This is derived by a contraction mapping theorem on the $L^{\infty}$-space on $(\Scal \times \Acal)_b$ \footnote{Recall the soft Bellman operator is a contraction mapping in $L^{\infty}$-space.  }.Here, recall 
 \begin{align*}
    \forall (s,a) \in (\Scal \times \Acal)_b;  q^{\star}_{\alpha}(s,a) = q_{\alpha}(s,a;0)\geq 0 
 \end{align*}
 The proof is performed as in the first step of the proof of \pref{lem:saddle}. We use the assumption in this step. 
 
 Furthermore, $ \forall (s,a) \in (\Scal \times \Acal)_b;q_{\alpha}(s,a;0)\leq q_{\alpha}(s,a;c)$ clearly holds, and, especially, the strict inequality holds on some point in $(\Scal \times \Acal)_b$. This is because if $q_{\alpha}(s,a;0)= q_{\alpha}(s,a;c)$, we get the contradiction:
\begin{align*}
      \forall (s,a) \in (\Scal \times \Acal)_b; 0 = \mathbb{E}_{s'\sim P(\cdot \mid s,a)}[\gamma \alpha \log(\sum_{a'} \exp(q(s',a')/\alpha)\pi_b(a'\mid s')+r-q(s,a)\mid s,a]=c(s,a). 
\end{align*}

However, in this situation, we have 
\begin{align*}
    \EE_{(s,a)\sim P_b}[q^2_{\alpha}(s,a;0) ]<   \EE_{(s,a)\sim P_b}[q^2_{\alpha}(s,a;c) ]. 
\end{align*}
Thus, it contradicts $q^2_{\alpha}(s,a;c)$ takes a least $L^2$-norm. Therefore, it is concluded $c(s,a)=0$ for any $(s,a) \in (\Scal \times \Acal)_b$. Thus, using the contraction mapping theorem, this implies the least $L^2$-norm solution is unique and $q^{\star}_{\alpha}(s,a)$ on $(s,a) \in (\Scal \times \Acal)_b$.

\end{proof}

\section{Proof in \pref{sec:identification}}

\subsection{Proof of \pref{lem:adjoint}}

We define an $L^2$-space $H'$ with no weight on the support of $(\Scal \times \Acal)_b$. In other words, we define 
\begin{align*}
    \langle q_1, q_2 \rangle_{\Hcal'} &\coloneqq \int q_1(s,a)q_2(s,a)\mathrm{I}(P_b(s,a)>0)\mathrm{d}(s,a). 
\end{align*}
Then, we have 
\begin{align*}
   \langle  l^{\star}_{\alpha} , (I- \gamma P^{\star}_{\alpha}) q \rangle_{\Hcal}  &= \langle  P_b l^{\star}_{\alpha} , (I - \gamma P^{\star}_{\alpha}) q \rangle_{\Hcal'}  \\ 
    &= \langle  (I-\gamma \{P^{\star}_{\alpha}\}^{\top})^{-1} (P_bq^{\star}_{\alpha}),(I - \gamma P^{\star}_{\alpha}) q \rangle_{\Hcal'} \\
    &= \langle  (I-\gamma \{P^{\star}_{\alpha}\}^{\top})(I-\gamma \{P^{\star}_{\alpha}\}^{\top})^{-1}(P_b q^{\star}_{\alpha}), q \rangle_{\Hcal'} \\
    &= \langle P_b q^{\star}_{\alpha} ,q \rangle_{\Hcal'}= \langle q^{\star}_{\alpha}, q \rangle_{\Hcal}. 
\end{align*}
Here, letting $P^{\star}_{\alpha}(s',a' \mid s,a)=P(s'\mid s,a)\pi^{\star}(a'\mid s')$ we use 
\begin{align*}
    &\langle g_1,  P^{\star}_{\alpha}g_2 \rangle_{\Hcal'}\\
    &=\int g_1(s,a)g_2(s',a')P^{\star}_{\alpha}(s',a' \mid s,a)\mathrm{I}(P_b(s,a)>0)\mathrm{d}\mu(s,a,s',a') \tag{Definition} \\ 
    &=\int g_1(s,a)g_2(s',a')P^{\star}_{\alpha}(s',a' \mid s,a)\mathrm{I}(P_b(s,a)>0)\mathrm{I}(P_b(s',a')>0)\mathrm{d}\mu(s,a,s',a') \tag{Recall we assume $\|\pi^{\star}/\pi_b\|\leq \infty$. } \\ 
    &=\langle \{P^{\star}_{\alpha}\}^{\top}g_1, g_2 \rangle_{\Hcal'}. 
\end{align*}
Note $\langle g_1,  P^{\star}_{\alpha}g_2 \rangle_{\Hcal}=\langle \{P^{\star}_{\alpha}\}^{\top}g_1, g_2 \rangle_{\Hcal}$ does \emph{not} generally hold.

\subsection{Proof of \pref{lem:saddle} } 
It is obvious that 
\begin{align*}
   \forall l\in \Hcal; L_{\alpha}(q^{\star}_{\alpha},l^{\star}_{\alpha})= L_{\alpha}(q^{\star}_{\alpha},l).
\end{align*}
Hence, we show 
\begin{align*}
     \forall q\in \Hcal; L_{\alpha}(q,l^{\star}_{\alpha}) \geq L_{\alpha}(q^{\star}_{\alpha},l^{\star}_{\alpha})
\end{align*}

To achieve this goal, we show 
\begin{align}\label{eq:important}
0.5\EE_{(s,a)\sim P_b}[(q-q^{\star}_{\alpha})(s,a)^2]\leq L_{\alpha}(q,l^{\star}_{\alpha})-L_{\alpha}(q^{\star}_{\alpha},l^{\star}_{\alpha}). 
\end{align} 
Using \pref{lem:adjoint}, a function $l^{\star}_{\alpha}$ satisfies a key adjoint property: 
\begin{align}\label{eq:adjoint_propety}
 \EE_{(s,a,s')\sim P_b,a'\sim \pi^{\star}_{\alpha}(s
')}[l^{\star}_{\alpha}(s,a)\{\gamma q(s',a')-q(s,a) \} ]= -\EE_{(s,a,s')\sim P_b,a'\sim \pi^{\star}_{\alpha}(s
')}[q^{\star}_{\alpha}(s,a)q(s,a) ]. 
\end{align}
for any $q \in \Hcal$. 

\paragraph{First Step: Show $l^{\star}_{\alpha}(s,a)\geq 0$ and  $q^{\star}_{\alpha}(s,a)\geq 0$ for any $(s,a)\in \Scal \times \Acal$.}

Recall 
\begin{align*}
    q^{\star}_{\alpha}(s,a)=\EE_{\pi^{\star}_{\alpha}}[r_0+\sum_{t=1}^{\infty}\{\gamma^t r_t -\alpha \gamma^t \log \pi^{\star}_{\alpha}(a_t\mid s_t)/\pi_b(a_t\mid s_t)\}\mid s_0=s,a_0=a]
\end{align*}
where the inside of the expectation is taken under a policy $\pi^{\star}_{\alpha}$. Hence, 
\begin{align}
    q^{\star}_{\alpha}(s,a)\geq \gamma(1-\gamma)^{-1}(R_{\min} - \alpha \log(\|\pi^{\star}_{\alpha}/\pi_b\|_{\infty})). 
\end{align}
Therefore, using Assumption~\ref{assum:scale}, we have 
$$q^{\star}_{\alpha}(s,a)\geq 0$$ for any $(s,a)\in \Scal \times \Acal$. 
Finally, from the definition of $l^{\star}_{\alpha}$, we obtain $$l^{\star}_{\alpha}(s,a)\geq 0$$ for any $(s,a)\in \Scal \times \Acal$. 

\paragraph{Second Step: Show $0.5\EE_{(s,a)\sim P_b}[(q-q^{\star}_{\alpha})(s,a)^2]\leq L_{\alpha}(q,l^{\star}_{\alpha})-L_{\alpha}(q^{\star}_{\alpha},l^{\star}_{\alpha})$ for any $q \in \Hcal$. }

Then, for any $q \in \Hcal$, we have 
\begin{align*}
   &L_{\alpha}(q,l^{\star}_{\alpha})- L_{\alpha}(q^{\star}_{\alpha},l^{\star}_{\alpha})\\
   &=0.5 \EE_{(s,a) \sim P_b}[q^2(s,a)-q^{\star}_{\alpha}(s,a)^2] + \\
   &+ \EE_{(s,a,s')\sim P_b}[l^{\star}_{\alpha}(s,a)\{\gamma \alpha \log \sum_{a'}\{\exp(q(s',a')/\alpha)\pi_b(a'\mid s')\}-q(s,a) \}  ] \\
   &-\EE_{(s,a,s')\sim P_b}[l^{\star}_{\alpha}(s,a)\{\gamma \alpha\log \sum_{a'} \{ \exp(q^{\star}_{\alpha}(s',a')/\alpha)\pi_b(a'\mid s')\}-q^{\star}_{\alpha}(s,a) \}  ]  \\
   &\geq 0.5 \EE_{(s,a) \sim P_b}[q^2(s,a)-q^{\star}_{\alpha}(s,a)^2]  + \\
   &+\EE_{(s,a,s')\sim P_b}\left[l^{\star}_{\alpha}(s,a) \gamma \frac{\sum_{a'} \exp(q^{\star}_{\alpha}(s',a')/\alpha)\pi_b(a'\mid s')\{q(s',a')-q^{\star}_{\alpha}(s',a')\}}{\sum_{a'} \exp(q^{\star}_{\alpha}(s',a')/\alpha)\pi_b(a'\mid s')  } \right]  \tag{Convexity }  \\
   &+\EE_{(s,a) \sim P_b}[-l^{\star}_{\alpha}(s,a)q(s,a) + l^{\star}_{\alpha}(s,a)q^{\star}_{\alpha}(s,a) ] \\
   &=0.5 \EE_{(s,a) \sim P_b}[q^2(s,a)-q^{\star}_{\alpha}(s,a)^2]+  \\
   &+ \EE_{(s,a,s')\sim P_b,a'\sim \pi^{\star}_{\alpha}(s
')}[l^{\star}_{\alpha}(s,a)\{\gamma q(s',a')-q(s,a) - \gamma q^{\star}_{\alpha}(s',a') + q^{\star}_{\alpha}(s,a) \} ]. 
\end{align*}
In the convexity part, what we use is 
\begin{align*}
   f(x)-f(y)\geq \sum_{i=1}^{|\Acal|} \frac{\partial f}{\partial y_i}(x_i-y_i),\quad f(x)=\alpha \log(\sum_{i=1}^{|\Acal|} \exp(x_i/\alpha)). 
\end{align*}
and $$\forall (s,a)\in (\Scal \times \Acal)_b; l^{\star}_{\alpha}(s,a)\geq 0.$$ Finally, by using an adjoint propety of $l^{\star}_{\alpha}$ in \pref{eq:adjoint_propety}, 
\begin{align*}
 & L_{\alpha}(q,l^{\star}_{\alpha})- L_{\alpha}(q^{\star}_{\alpha},l^{\star}_{\alpha})  \\ 
 & \geq 0.5 \EE_{(s,a) \sim P_b}[q^2(s,a)-q^{\star}_{\alpha}(s,a)^2]-\EE_{(s,a) \sim P_b}[q^{\star}_{\alpha}(s,a)\{q(s,a)-q^{\star}_{\alpha}(s,a)\}] \\ 
   &=0.5 \EE_{(s,a) \sim P_b}[\{q(s,a)-q^{\star}_{\alpha}(s,a) \}^2]. 
\end{align*}
Then, \pref{eq:important} is concluded. 

\subsection{Proof of \pref{lem:saddle2}  }   

We denote the solution of $\argmin_{q \in \Qcal}\sup_{l \in \Lcal}L_{\alpha}(q,l)$ by $\hat q$. Let 
\begin{align*}
    \hat l(q) = \argmax_{l \in \Lcal}L_{\alpha}(q,l), \quad \hat l= \hat l(\hat q). 
\end{align*}
Then, we have 
\begin{align*}
    &L_{\alpha}(\hat q,l^{\star}_{\alpha})-L_{\alpha}( q^{\star}_{\alpha},l^{\star}_{\alpha})\\ 
    &=  \underbrace{-L_{\alpha}( q^{\star}_{\alpha},l^{\star}_{\alpha}) + L_{\alpha}(q^{\star}_{\alpha},\hat l(q^{\star}_{\alpha}))}_{(a)}\underbrace{-  L_{\alpha}(q^{\star}_{\alpha}, \hat l(q^{\star}_{\alpha})) +  L_{\alpha}(\hat q, \hat l)}_{(c)} \underbrace{- L_{\alpha}(\hat q, \hat l)+ L_{\alpha}(\hat q, l^{\star}_{\alpha})}_{(d)} \\
    &\leq 0. 
\end{align*}
In (a), we use the property of the saddle point in \pref{lem:saddle}. In (c) and (e), we use the definition of estimators. Hence, 
\begin{align*}
   0&\leq 0.5 \EE_{(s,a) \sim P_b}[\{q(s,a)-q^{\star}_{\alpha}(s,a) \}^2]  \\ 
   &\leq L_{\alpha}(\hat q,l^{\star}_{\alpha})-L_{\alpha}( q^{\star}_{\alpha},l^{\star}_{\alpha}) \tag{Recall \pref{eq:important}} \\
   &\leq 0. 
\end{align*}
This concludes the statement. 

\subsection{Proof of \pref{thm:convergence_soft} }

In this proof, the expectation is taken with respect to the offline data. We define 
\begin{align*}
\hat L_{\alpha}(q,w) &= \EE_n[0.5 q^2(s,a)+ l(s,a)\{r+\alpha \gamma \log(\sum_{a'} \exp(q(s',a')/\alpha)\pi_b(a'\mid s') )-q(s,a)\}], \\
\hat l(q) &= \argmax_{l\in \Lcal}\EE_n[0.5 q^2(s,a)+ l(s,a)\{r+\alpha \gamma \log(\sum_{a'}\exp(q(s',a')/\alpha)\pi_b(a'\mid s') )-q(s,a)\}],\\
\hat l &= \hat l(\hat q_{\alpha}). 
\end{align*}
We use the following decomposition:
\begin{align*}
    L_{\alpha}(\hat q_{\alpha},l^{\star}_{\alpha})&=  \underbrace{-L_{\alpha}( q^{\star}_{\alpha},l^{\star}_{\alpha}) + L(q^{\star}_{\alpha},\hat l(q^{\star}_{\alpha}))}_{(a)} \underbrace{-L_{\alpha}(q^{\star}_{\alpha}, \hat l(q^{\star}_{\alpha}))+\hat L_{\alpha}(q^{\star}_{\alpha}, \hat l(q^{\star}_{\alpha}))}_{(b)} \\ 
    &\underbrace{- \hat L_{\alpha}(q^{\star}_{\alpha}, \hat l(q^{\star}_{\alpha})) + \hat L_{\alpha}(\hat q_{\alpha}, \hat l)}_{(c)} \underbrace{-\hat L_{\alpha}(\hat q_{\alpha}, \hat l)+ \hat L_{\alpha}(\hat q_{\alpha}, l^{\star}_{\alpha})}_{(d)}\underbrace{-\hat L_{\alpha}(\hat q_{\alpha}, l^{\star}_{\alpha})+ L_{\alpha}(\hat q_{\alpha},l^{\star}_{\alpha})}_{(e)}. 
\end{align*}
Here, terms (a) is than $0$ using the saddle point property in \pref{lem:saddle}. Term (c) and (d) are less than $0$ using the construction of estimators. Besides, using Hoeffding's inequality, with probability $1-\delta$, terms (b) and (e) are less than $$c \prns{  \Bcal^2_{\Qcal}+\Bcal_{\Qcal}\Bcal_{\Lcal} + \gamma \Bcal_{\Lcal}\{\alpha+\alpha \ln(|\Acal|)\}} \sqrt{\frac{\ln(|\Qcal||\Lcal|/\delta)}{n}}.$$ 
Hereafter, we condition on this event. 
Here, we use 
\begin{align*}
    &|0.5 q^2(s,a)+ l(s,a)\{r+\alpha \gamma \log(\sum_{a'} \exp(q(s',a')/\alpha)\pi_b(a'\mid s') )-q(s,a)\}| \\ 
    &\leq  |\alpha l(s,a)\gamma \log(\sum_{a'} \exp(q(s',a')/\alpha)\pi_b(a'\mid s')) | + |0.5 q^2(s,a) + l(s,a)\{r-q(s,a)\}|\\
    &\leq 0.5 \Bcal^2_{\Qcal}+\Bcal_{\Lcal}\{R_{\max}+\Bcal_{\Qcal}\}+ \gamma \Bcal_{\Lcal}\alpha \max_{a'}(\Bcal_{\Qcal}/\alpha +\pi_b(a'\mid s') )+\gamma \Bcal_{\Lcal} \ln(|\Acal|) \tag{Use \pref{lem:useful_inequality} and $l(s,a)\geq 0$ for any $(s,a)$.}\\
    &\lesssim \Bcal^2_{\Qcal}+ \Bcal_{\Lcal}\Bcal_{\Qcal}+  \gamma \Bcal_{\Lcal}\{ \alpha+\alpha \ln(|\Acal|)\}. 
\end{align*}

Therefore, 
\begin{align*}
& 0.5\EE_{(s,a) \sim P_b}[\{\hat q_{\alpha}(s,a)-q^{\star}_{\alpha}(s,a) \}^2] \\
&\leq L(l^{\star}_{\alpha},\hat q_{\alpha})-L(l^{\star}_{\alpha},q^{\star}_{\alpha})  \tag{Recall \pref{eq:important}}\\
& \leq c \prns{  \Bcal^2_{\Qcal}+ \Bcal_{\Lcal}\Bcal_{\Qcal}+ \Bcal_{\Lcal}\gamma \{\alpha+\alpha  \ln(|\Acal|)\}} \sqrt{\frac{\ln(|\Qcal||\Lcal|/\delta)}{n}}. 
\end{align*}
In the first inequality, we use \pref{eq:important}. 

\subsection{Proof of \pref{thm:convergence_q}}

In this proof, the expectation is taken w.r.t. the offline data. We define 
\begin{align*}
    L_0(q,l) &= \EE[0.5 q^2(s,a)+ l(s,a)\{r+\gamma \max_{a'}(q(s',a')- q(s,a)\}], \\  
\hat L_0(q,l) &= \EE_n[0.5 q^2(s,a)+ l(s,a)\{r+\gamma \max_{a'} q(s',a') -q(s,a)\}], \\
\hat l(q) &= \argmax_{l}\EE_n[0.5 q^2(s,a)+ l(s,a)\{r+\max_{a'}q(s',a')-q(s,a)\}],\\
\hat l &= \hat l(\hat q_0). 
\end{align*}

\paragraph{Part1: Show $l^{\star}(s,a)\geq 0$ for any $(s,a) \in \Scal \times \Acal$.}

First, since $R_{\min}\geq 0$, we have $q^{\star}(s,a)\geq 0$. Then, recalling the definition:
$$l^{\star} \coloneqq \{(I-\gamma \{P^{\star}\}^{\top})^{-1 }(q^{\star}P_{\pi_b})\}/P_{\pi_b},$$ 
we obtain $l^{\star}(s,a)\geq 0$ for any $(s,a) \in \Scal \times \Acal$. 

\paragraph{Part 2: Showing $ 0.5\EE_{(s,a)\sim P_b}[(q-q^{\star})(s,a)^2]\leq L_0(q,l^{\star})-L_0(q^{\star},l^{\star}). $ } 

Note $l^{\star}(s,a)$ satisfies 
\begin{align}\label{eq:adjoint2}
 \EE_{(s,a,s')\sim P_b,a'\sim \pi^{\star}(s
')}[l^{\star}(s,a)\{\gamma q(s',a')-q(s,a) \} ]= -\EE_{(s,a,s')\sim P_b,a'\sim \pi^{\star}(s
')}[q^{\star}(s,a)q(s,a) ]. 
\end{align}
for any $q(s,a) \in \Qcal$. 

Then, for any $q \in \Qcal$, we have 
\begin{align*}
   &L_0(q,l^{\star})- L_0(q^{\star},l^{\star})\\
   &=0.5 \EE_{(s,a) \sim P_b}[q^2(s,a)-q^{\star}(s,a)^2] +\EE_{(s,a,s')\sim P_b}[l^{\star}(s,a)\{\gamma \max_{a'} q(s',a')-q(s,a) \}  ] \\
   &-\EE_{(s,a,s')\sim P_b}[l^{\star}(s,a)\{\gamma  \max_{a'} q^{\star}(s',a')-q^{\star}(s,a) \}  ]  \\
   &\geq 0.5 v\EE_{(s,a) \sim P_b}[q^2(s,a)-q^{\star}(s,a)^2]  +\EE_{(s,a,s')\sim P_b}[l^{\star}(s,a) \{\gamma q(s',\pi^{\star}(s))-\gamma q^{\star}(s',\pi^{\star}(s))\} ]  \tag{Convexity }  \\
   &+\EE_{(s,a) \sim P_b}[-l^{\star}(s,a)q(s,a) + l^{\star}(s,a)q^{\star}(s,a) ] \\
   &=0.5 \EE_{(s,a) \sim P_b}[q^2(s,a)-q^{\star}(s,a)^2]+  \\
   &+ \EE_{(s,a,s')\sim P_b,a'\sim \pi^{\star}(s
')}[l^{\star}(s,a)\{\gamma q(s',a')-q(s,a) - \gamma q^{\star}(s',a') + q^{\star}(s,a) \} ]. 
\end{align*}
In the convexity part, what we use is 
\begin{align*}
   \max_i \{f(x_i)\}- \max_{i} \{f(y_i)\} \geq f(x_{\argmax_{i} \{f(y_i)\} })-\max_{i} \{f(y_i)\} 
\end{align*}
and 
\begin{align*}
   \forall (s,a)\in (\Scal \times \Acal)_b; l^{*}(s,a)\geq 0.
\end{align*}
Finally, by using the adjoint property \pref{eq:adjoint2}, 
\begin{align*}
 &L_0(q,l^{\star})- L_0(q^{\star},l^{\star})\\
  &\geq 0.5 \EE_{(s,a) \sim P_b}[q^2(s,a)-q^{\star}(s,a)^2]-\EE_{(s,a) \sim P_b}[q^{\star}(s,a)\{q(s,a)-q^{\star}(s,a)\}] \\ 
   &=0.5 \EE_{(s,a) \sim P_b}[\{q(s,a)-q^{\star}(s,a) \}^2]. 
\end{align*}

\paragraph{Part 3: Showing the final bound.}

We use the following decomposition: 
\begin{align*}
    L_0(\hat q_0,l^{\star})&=  \underbrace{-L( q^{\star},l^{\star}) + L_0(q^{\star},\hat l(q^{\star}))}_{(a)} \underbrace{-L_0(q^{\star}, \hat l(q^{\star}))+\hat L_0(q^{\star}, \hat l(q^{\star}))}_{(b)} \\ 
    &\underbrace{- \hat L_0(q^{\star}, \hat l(q^{\star})) + \hat L_0(\hat q_0, \hat l)}_{(c)} \underbrace{-\hat L_0(\hat q_0, \hat l)+ \hat L_0(\hat q_0, l^{\star})}_{(d)}\underbrace{-\hat L_0(\hat q_0, l^{\star})+ L_0(\hat q_0,l^{\star})}_{(e)}. 
\end{align*}

Here, (a) is equal to $0$ since $q^{\star}$ satisfies the Bellman equation. Besides, using Hoeffding's inequality, with probability $1-\delta$, terms (b) and (e) are less than $$c \prns{  \Bcal^2_{\Qcal}+ \Bcal_{\Lcal}\Bcal_{\Qcal}} \sqrt{\frac{\ln(|\Qcal||\Lcal|/\delta)}{n}}.$$ Terms (c) and (d) are greater than $0$ using the construction of estimators. 
This implies 
\begin{align*}
0.5\EE_{(s,a) \sim P_b}[\{\hat q_0(s,a)-q^{\star}(s,a) \}^2]
&\leq L_0(\hat q_0,l^{\star})- L_0(q^{\star},l^{\star}) \tag{Second step.}\\
&\leq c\prns{  \Bcal^2_{\Qcal}+  \Bcal_{\Lcal}\Bcal_{\Qcal}}\sqrt{\frac{\ln(|\Qcal||\Lcal|/\delta)}{n}}. 
\end{align*}

\section{Proof in \pref{sec:theory}}\label{sec:detail_proof}

\subsection{Proof of \pref{thm:main}}

Using \pref{thm:convergence_soft}, with probability $1-\delta$, the following holds 
\begin{align}\label{eq:l2_error_softmax}
    &\EE_{(s,a) \sim P_b}[\{\hat q_{\alpha}(s,a)-q^{\star}_{\alpha}(s,a)\}^2 ] \lesssim \mathrm{Error},  \\
     & \mathrm{Error} =\prns{  \Bcal^2_{\Qcal}+ \Bcal_{\Qcal}\Bcal_{\Lcal}+ \Bcal_{\Lcal}\gamma \{\alpha+\alpha  \ln(|\Acal|)\}} \sqrt{\frac{\ln(|\Qcal||\Lcal|/\delta)}{n}}.  \nonumber
\end{align}
Hereafter, we condition on this event. Then, letting $\hat \pi_{\alpha}(a\mid s)=\mathrm{softmax}(\hat q_{\alpha}/\alpha + \log \pi_b)$, we have  
\begin{align*}
    J(\pi^{\star}_{\alpha})-J(\hat \pi_{\alpha})&\leq (1-\gamma)^{-1}R_{\max} \EE_{s \sim d_{\pi^{\star}_{\alpha}}}[\sum_{a} |\pi^{\star}_{\alpha}(a\mid s)-\hat \pi_{\alpha}(a \mid s)| ] \tag{Performance difference lemma in \pref{lem:performance}} \\
     &\leq (1-\gamma)^{-1}R_{\max}\EE_{s \sim d_{\pi^{\star}_{\alpha}}}\bracks{\sqrt{|\Acal|}\prns{\sum_{a}\{\pi^{\star}_{\alpha}(a\mid s)-\hat \pi_{\alpha}(a\mid s)\}^2}^{1/2}} \tag{CS inequality} \\ 
    &= (1-\gamma)^{-1}R_{\max}\EE_{s \sim d_{\pi^{\star}_{\alpha}}}\bracks{\sqrt{|\Acal|}\prns{\sum_{a: \pi_b(a \mid s)>0}\{\pi^{\star}_{\alpha}(a\mid s)-\hat \pi_{\alpha}(a\mid s)\}^2}^{1/2}}.
\end{align*}
From the second line to the third line, we use a relation that $\pi^{\star}_{\alpha},\hat \pi_{\alpha}$, and $\pi_b$ have the same support. 

Now, we want to connect this bound with $L^2$-error of Q-functions. This is possible since the softmax function with Lipschitz constant $1/\alpha$ is Lipschitz continuous with constant $1/\alpha$ in \citet[Proposition 4]{gao2017properties}. Hence,  the right hand side is upper-bounded by 
\begin{align*}
(1-\gamma)^{-1}\sqrt{|\Acal|}/\alpha R_{\max} \EE_{s \sim d_{\pi^{\star}_{\alpha}}}\bracks{\braces{\sum_{a: \pi_b(a \mid s)>0}\{\hat q_{\alpha}(s,a)-q^{\star}_{\alpha}(s,a)\}^2}^{1/2} }. 
\end{align*}
From Jensen's inequality, this is further upper-bounded by 
\begin{align*}
  (1-\gamma)^{-1}\sqrt{|\Acal|}/\alpha R_{\max} \EE_{s \sim d_{\pi^{\star}_{\alpha}}}\bracks{\sum_{a: \pi_b(a \mid s)>0}\{\hat q_{\alpha}(s,a)-q^{\star}_{\alpha}(s,a)\}^2}^{1/2}.  
\end{align*}
Then, using the definition of $\pi^{\diamond}_b$, this is upper-bounded by 
\begin{align*}
 (1-\gamma)^{-1}|\Acal|/\alpha R_{\max} \EE_{s \sim d_{\pi^{\star}_{\alpha}},a\sim \pi^{\diamond}_b(\cdot \mid s)}\bracks{\{\hat q_{\alpha}(s,a)-q^{\star}_{\alpha}(s,a)\}^2}^{1/2} 
\end{align*}
Finally, by combining this result with \pref{eq:l2_error_softmax}, we get the final guarantee.

\subsection{Proof of Theorem~\ref{cor:main}}

We have  
\begin{align}\label{eq:helpful_inequality}
    J(\pi)-J_{\alpha}(\pi)=R_{\max} \alpha (1-\gamma)^{-1}\EE_{s\sim d_{\pi},a\sim \pi(s)}[\log \pi(a\mid s)/\pi_b(a\mid s)]\geq 0
\end{align}
for any $\pi$. Then, 
\begin{align*}
    J(\pi^{\star})-J(\pi^{\star}_{\alpha}) &=J(\pi^{\star})-J_{\alpha}(\pi^{\star})+\underbrace{J_{\alpha}(\pi^{\star})-J_{\alpha}(\pi^{\star}_{\alpha})}_{(a)}+\underbrace{J_{\alpha}(\pi^{\star}_{\alpha})-  J(\pi^{\star}_{\alpha})}_{(b)}\\
    &\leq J(\pi^{\star})-J_{\alpha}(\pi^{\star}). 
\end{align*}
Note the term (a) is less than $0$ since $\pi^{\star}_{\alpha}$ is the optimal softmax policy and the term (b) is less than $0$ using \pref{eq:helpful_inequality}. Furthermore, 
\begin{align*}
    & J(\pi^{\star})-J_{\alpha}(\pi^{\star})= (1-\gamma)^{-1}\alpha R_{\max} \EE_{s\sim d_{\pi^{\star}},a\sim \pi^{\star}(s)}[\log \pi^{\star}(a\mid s)/\pi_b(a\mid s)] \\
    &\leq (1-\gamma)^{-1}\alpha R_{\max} \log C_0. 
\end{align*}
Therefore, by combining with \pref{thm:main}, we have
\begin{align*}
    J(\pi^{\star})-J(\hat \pi_{\alpha})&\leq c R_{\max}\frac{|\Acal| C^{1/2}_{\Qcal,d_{\pi^{\star}_{\alpha},\mu_0} }\{\Bcal_{\Qcal} + \gamma \Bcal_{\Lcal}\{\Bcal_{\Qcal} 
 \alpha +\ln(|\Acal|)\}\}^{1/2} 
    \{\ln(|\Qcal||\Lcal|/\delta)\}^{1/4}}{(1-\gamma)\alpha n^{1/4}}  \\
    &+ c\frac{R_{\max} \alpha \log C_0}{1 -\gamma}. 
\end{align*}
The sample complexity is easily obtained from this result.

\section{ Proof in \pref{sec:margin}}

From \pref{sec:detail_proof}, with probability $1-\delta$, the following holds 
\begin{align}\label{eq:l2_error}
    &\EE_{(s,a) \sim P_b}[\{\hat q(s,a)-q^{\star}(s,a)\}^2 ]\lesssim \mathrm{Error}, \\ 
    &\mathrm{Error}=\prns{  \Bcal^2_{\Qcal}+ \gamma \Bcal_{\Lcal}\Bcal_{\Qcal}} \sqrt{\frac{\ln(|\Qcal||\Lcal|/\delta)}{n}}. \nonumber
\end{align}
Hereafter, we condition on this event. Then, letting $$\hat \pi(a\mid s)=\argmax_{a \in \Acal:\pi_b(a\mid s)>0} \hat q(s,a),$$ we have  
\begin{align*}
    &(1-\gamma)^2 R_{\max}^{-1}\{J(\pi^{\star})-J(\hat \pi)\} \\
    &\leq \EE_{s \sim d_{\pi^{\star}}}[\mathrm{I}(\pi^{\star}(s)\neq \hat \pi(s)) ] \tag{Performance difference lemma in \pref{lem:performance}} \\
     &\leq \EE_{s \sim d_{\pi^{\star}}}\bracks{\sum_{a':a'\sim \pi_b(a\mid s)>0} \mathrm{I}(\hat q(s,a')-\hat q(s,\pi^{\star}(s))\geq 0\,\&\,q^{\star}(s,a')-q^{\star}(s,\pi^{\star}(s))< 0 )  }. 
\end{align*}
Hence, we have 
\begin{align*}
    &\EE_{s \sim d_{\pi^{\star}}}\bracks{\sum_{a':a'\sim \pi_b(a\mid s)>0} \mathrm{I}(\hat q(s,a')-\hat q(s,\pi^{\star}(s))\geq 0\,\&\,q^{\star}(s,a')-q^{\star}(s,\pi^{\star}(s))< 0 )  } \\
    &\leq \EE_{s \sim d_{\pi^{\star}}}\bracks{\sum_{a':a'\sim \pi_b(a\mid s)>0} \mathrm{I}(0> q^{\star}(s,a')-q^{\star}(s,\pi^{\star}(s))\geq -t )  } + \\
     & + \EE_{s \sim d_{\pi^{\star}}}\bracks{\sum_{a':a'\sim \pi_b(a\mid s)>0} \mathrm{I}(\hat q(s,a')-\hat q(s,\pi^{\star}(s))-q^{\star}(s,a')+q^{\star}(s,\pi^{\star}(s))\geq t )  }
\end{align*}
In the first term, we can use a margin assumption:
\begin{align*}
    \EE_{s \sim d_{\pi^{\star}}}\bracks{\sum_{a':a'\sim \pi_b(a\mid s)>0} \mathrm{I}(0> q^{\star}(s,a')-q^{\star}(s,\pi^{\star}(s))\geq -t )  }\leq c|\Acal|(t/t_0)^{\beta}. 
\end{align*}
In the second term, we can use 
\begin{align*}
    &\EE_{s \sim d_{\pi^{\star}}}\bracks{\sum_{a':a'\sim \pi_b(a\mid s)>0} t^2 \mathrm{I}(\hat q(s,a')-\hat q(s,\pi^{\star}(s))-q^{\star}(s,a')+q^{\star}(s,\pi^{\star}(s))\geq t )  } \\
    &\leq \EE_{s \sim d_{\pi^{\star}}}\bracks{\sum_{a':a'\sim \pi_b(a\mid s)>0} |\hat q(s,a')-\hat q(s,\pi^{\star}(s))-q^{\star}(s,a')+q^{\star}(s,\pi^{\star}(s))|^2 } \\
  &\leq 2\EE_{s \sim d_{\pi^{\star}}}\bracks{\sum_{a':a'\sim \pi_b(a\mid s)>0} |\hat q(s,a')-q^{\star}(s,a')|^2_2 + |\hat q(s,\pi^{\star}(s))-q^{\star}(s,\pi^{\star}(s))|^2 }\\ 
  &\leq 2|\Acal| \EE_{s \sim d_{\pi^{\star}},a\sim \pi^{\diamond}_b(s)}[|\hat q(s,a')-q^{\star}(s,a')|^2] +2|\Acal| \EE_{s \sim d_{\pi^{\star}},a\sim \pi^{\star}(s)}[|\hat q(s,a)-q^{\star}(s,a)|^2]. 
\end{align*}
Therefore, 
\begin{align*}
&(1-\gamma)^2 R_{\max}^{-1}\{J(\pi^{\star})-J(\hat \pi)\}\\
&\leq c_2 \{ |\Acal|(t/t_0)^{\beta} +t^{-2}\{ 2 |\Acal| \EE_{s \sim d_{\pi^{\star}},a\sim \pi^{\diamond}_b(s)}[|\hat q(s,a')-q^{\star}(s,a')|^2]  \\
& +2|\Acal| \EE_{s \sim d_{\pi^{\star}},a\sim \pi^{\star}(s)}[|\hat q(s,a)-q^{\star}(s,a)|^2]\} \}  \\
&\leq c_3 \{ |\Acal| t^{-2\beta/(2+\beta)}_0\{\EE_{s \sim d_{\pi^{\star}},a\sim \pi^{\diamond}_b(s)}[|\hat q(s,a')-q^{\star}(s,a')|^2]^{\beta/(2+\beta)} \\ 
&+  \EE_{s \sim d_{\pi^{\star}},a\sim \pi^{\star}(s)}[|\hat q(s,a)-q^{\star}(s,a)|^2]^{\beta/(2+\beta)}  \} \}. 
\end{align*}
Hence, 
\begin{align*}
    (1-\gamma)^2R_{\max}^{-1}\{J(\pi^{\star})- J(\hat \pi)\}\leq c|\Acal|t^{-2\beta/(2+\beta)}_0\prns{C_{\Qcal,d_{\pi^{\star}_0,\mu_0} }C_0\prns{  \Bcal^2_{\Qcal}+ \gamma \Bcal_{\Lcal}\Bcal_{\Qcal}} 
    \{\ln(|\Qcal||\Lcal/\delta)/n\}^{1/2}    }^{\beta/(2+\beta)}. 
\end{align*}
This concludes the statement by some algebra.

\section{Auxiliary Lemmas}

We prove two auxiliary lemmas used in the proof. 

\begin{lemma}[Performance Difference Lemma ]\label{lem:performance}
\begin{align*}
    J(\pi)-J(\pi')&=(1-\gamma)^{-1}\EE_{s \sim d_{\pi}}[\langle \pi(\cdot \mid s)-\pi'(\cdot \mid s),Q^{\pi'}(s,\cdot)\rangle  ] \\
      &\leq (1-\gamma)^{-2}R_{\max} \EE_{s \sim d_{\pi}}[ \|\pi(\cdot \mid s)-\pi'(\cdot \mid s)\|_1 ]. 
\end{align*}    
\end{lemma}

\begin{lemma}[LogSumExp Inequality]\label{lem:useful_inequality}
    \begin{align*}
        \max\{x_1,\cdots,x_{|\Acal|}\} \leq \log(\sum_{i=1}^{|\Acal|} \exp(x_i))\leq  \max\{x_1,\cdots,x_{|\Acal|}\} + \ln(|\Acal|). 
    \end{align*}
\end{lemma}

\end{document}